\documentclass[11pt]{article}

\usepackage{geometry}
\usepackage{natbib}

\usepackage{pgfplots}

\usepackage{microtype}
\usepackage{graphicx}
\usepackage{booktabs}
\usepackage{array}

\usepackage{hyperref}

\usepackage{color}
\usepackage{booktabs}       
\usepackage{amsfonts}       
\usepackage{nicefrac}       
\usepackage{microtype}      
\usepackage{lipsum}
\usepackage{bbm}
\usepackage{dsfont}

\usepackage{pgffor}
\usepackage{bbm}
\usepackage{graphicx}
\usepackage{subcaption}
\usepackage{multirow}
\usepackage[normalem]{ulem}

\usepackage{amsmath}

\newcommand{\comm}[1]{\textcolor{cyan}{}}

\usepackage{amsthm}

\newtheorem{definition}{Definition}
\newtheorem{proposition}{Proposition}
\newtheorem{theorem}{Theorem}
\newtheorem{remark}{Remark}
\newtheorem{lemma}{Lemma}

\theoremstyle{definition}
\newtheorem{example}{Example}

\newcommand{\alice}[0]{\textit{Alice}}
\newcommand{\bob}[0]{\textit{Bob}}
\newcommand{\eve}[0]{\textit{Eve}}
\newcommand{\sm}[0]{\textit{sm}}
\newcommand{\fr}[0]{\textit{fr}}

\newcommand{\wfomc}[0]{\operatorname{WFOMC}}

\title{Weighted First-Order Model Counting in the Two-Variable Fragment With Counting Quantifiers}

\author{%
   Ond\v{r}ej Ku\v{z}elka \\
   Faculty of Electrical Engineering \\
   Czech Technical University in Prague \\
   Prague, Czech Republic \\
}
\date{}

\begin{document}

\maketitle

\begin{abstract}
It is known due to the work of Van den Broeck et al [KR, 2014] that weighted first-order model counting (WFOMC) in the two-variable fragment of first-order logic can be solved in time polynomial in the number of domain elements. In this paper we extend this result to the two-variable fragment with counting quantifiers.
\end{abstract}

\section{Introduction}

In this paper we study {\em weighted first-order model counting} (WFOMC), which is an important problem (not only) because it can be used for probabilistic inference in most statistical relational learning models \citep{van2011lifted,getoor2007introduction}. Probabilistic inference is in general intractable and the same holds for probabilistic inference in relational domains and therefore also for WFOMC. {\em Lifted inference} refers to a set of methods developed in the probabilistic inference literature which exploit structure and symmetries of the problems for making inference more tractable, e.g.\ \citep{poole2003first,DBLP:conf/ijcai/BrazAR05,DBLP:conf/uai/GogateD11a,broeck2011completeness,van2014skolemization,DBLP:conf/nips/KazemiKBP16}. One of the most celebrated results on symmetric WFOMC comes from the works \citep{broeck2011completeness,van2014skolemization} which established that WFOMC can be solved in polynomial time for any fixed first-order logic sentence which contains at most two variables. In the lifted inference literature, problems which admit such polynomial-time algorithms are called {\em domain-liftable} \citep{broeck2011completeness}.\footnote{Like, among others, the works \citep{broeck2011completeness,van2014skolemization,kazemi2014relational,DBLP:conf/lics/KuusistoL18}, we also consider only the {\em symmetric} version of weighted first-order model counting. For details and differences with the asymmetric version, we refer to the paper \citep{beame2015symmetric}.}

\citet{DBLP:conf/lics/KuusistoL18} recently extended the domain-liftability result for the two-variable fragment by allowing to express one {\em functionality} constraint. That is, one can specify that some binary relation should behave as a function. In their paper, they also mentioned (although without giving any details) that if they could further extend their result to multiple funcionality constraints, they could also establish domain liftability of the two-variable fragment of first-order logic with counting quantifiers $\exists_{=k}$, $\exists_{\leq k}$ and $\exists_{\geq k}$ (see, e.g., \citep{gradel1997two}), which stand for {\em exist exactly k}, {\em exists at most k} and {\em exist at least k}, respectively. Motivated by the work of \citeauthor{DBLP:conf/lics/KuusistoL18}, in this paper, we first show a simpler method to add an arbitrary number of {\em function} constraints and {\em cardinality} constraints to sentences from the two-variable fragment while still guaranteeing polynomial-time inference. We then use this result to prove that WFOMC is domain-liftable for sentences from the two-variable fragment of first-order logic with counting quantifiers.

This paper is an extension of our preliminary report \citep{kuzelka2020lifted} in which we showed domain-liftability of WFOMC with functionality constraints in the context of Markov logic networks. 
The rest of the paper is structured as follows. Section \ref{sec:background} contains the background material needed for our technical results. In Sections \ref{sec:modelcountingfunction}-\ref{sec:exists-k-quantifier-section}, we work towards the proof of our main result which we give in Section \ref{sec:main-result}. We discuss related work in Section \ref{sec:related-work} and conclude the paper in Section \ref{sec:conclusions}. Appendix located at the end of the paper then contains omitted proofs and additional examples as well as some additional technical material that is not needed for the main result but which further illustrates some of our techniques.

\section{Background}\label{sec:background}



\subsection{Lagrange Interpolation}\label{sec:lagrange}

Lagrange interpolation (see, e.g., \cite{seroul2000programming}) is a classical method for finding the {\em unique} polynomial $p(x)$ of degree $d$ that, for given $d+1$ points $(x_0,y_0), (x_1,y_1)$, $\dots$, $(x_{d},y_{d})$ satisfies $p(x_0) = y_0$, $p(x_1) = y_1$, $\dots$, $p(x_{d}) = y_{d}$ (under the condition that $x_i \neq x_j$ for all $i \neq j$). There are various methods of finding the coefficients of the polynomial (e.g.\ based on special algorithms for solving systems of linear equations with Vandermonde matrices), however, in this paper it will be enough to consider the elementary method based on the explicit Lagrange formula:

$$L(x) = \sum_{i = 0}^d y_i \cdot l_i(x) $$

\noindent where $l_i$ is defined as

$$l_i(x) =  \prod_{\scriptsize{\begin{array}{c} 0 \leq j \leq d \\ i \neq j \end{array}}} \frac{x - x_j}{x_i - x_j}.$$

The next proposition shows a useful property of ``bit complexity'' of the coefficients of such interpolating polynomials that will be useful later in this paper (the proof of this proposition is located in the appendix).

\begin{proposition}\label{prop:lagrange}
Let $L(x) = \sum_{j=0}^{d} a_j \cdot x^j$ be the interpolating polynomial (written in the standard form as a sum of monomials) of points $(x_0,y_0)$, $(x_1,y_1)$, $\dots$, $(x_{d},y_{d})$, where all $x_j$'s are integers and all $y_j$'s are rational numbers, represented as fractions of integers. Every $a_j$ can be represented as a fraction $a_j = \frac{b_j}{c_j}$ and the number of bits needed to represent the integers $b_j$ and $c_j$ is polynomial in $d$ and in the number of bits needed to represent the points $(x_0,y_0)$, $(x_1,y_1)$, $\dots$, $(x_{d},y_{d})$.
\end{proposition}



\subsection{First-Order Logic}

We assume that the reader is familiar with first-order logic and we only cover it briefly in this section to set up notation used throughout the paper. 

We work with function-free first-order logic languages $\mathcal{L}$, defined by a set of constants, called domain and usually denoted as $\Delta$, a set of variables $\mathcal{V}$ and a set of predicates $\mathcal{R}$ (relations). When there is no risk of confusion, we assume such a language implicitly and do not specify its components $\mathcal{V}$, $\mathcal{R}$ (although we will usually specify the domain). We use $\textit{arity}(R)$ to denote the arity of a predicate~$R$. 
An expression of the form $r(a_1,...,a_k)$, with $a_1,...,a_k\in \Delta \cup \mathcal{V}$ and $r\in \mathcal{R}$, is called an {\em atom} or {\em atomic formula}. For example, $\sm(x)$, $\sm(\alice)$ and $\fr(\alice,y)$ are atoms. A variable which is not bound by any quantifier is called {\em free}. A first-order logic formula with no free variables is called a sentence. For instance, the formula $\forall x : \neg f(x,x)$ is a sentence, whereas the formula $f(x,x)$ is not a sentence as the variable $x$ is free in it. 
A first-order logic formula in which none of the atoms contains any variables is called {\em ground}. A possible world $\omega$ is represented as a set of ground atoms that are true in $\omega$. The satisfaction relation $\models$ is defined in the usual way: $\omega \models \alpha$ means that the formula $\alpha$ is true in $\omega$. For instance, if $\omega = \{ \sm(\bob) \}$ is a possible world on the domain $\Delta = \{ \alice, \bob \}$ then it holds $\omega \models (\exists x : \sm(x))$ and $\omega \not\models (\forall x : \sm(x))$.

The two-variable fragment of first-order logic (FO$^2$) is obtained by restricting the set of variables to $\mathcal{V} = \{x,y \}$. For example, the sentence $\forall x \forall y : a(x) \wedge e(x,y) \Rightarrow a(y)$ is in FO$^2$. The fragment of first-order logic FO$^2$ is interesting among others because (i) satisfiability is decidable for it (in particular it is NEXPTIME-complete) and (ii) weighted first-order model counting is polynomial-time (in the size of the domain) for any sentence from FO$^2$ \citep{broeck2011completeness,van2014skolemization}.

\subsubsection{First-Order Logic With Counting Quantifiers}

An interesting extension of the 2-variable fragment of first order logic is obtained by adding {\em counting quantifiers} $\exists_{=k}$, $\exists_{\leq k}$ and $\exists_{\geq k}$ to it \citep{gradel1997two}. Satisfiability in this fragment of first-order logic is still decidable, although this fragment lacks the finite-model property that FO$^2$ enjoys. 

The counting quantifiers can be introduced as follows. Let $\omega$ be a possible world defined on a domain $\Delta$. The sentence $\exists_{\geq k} x : \psi(x)$ is true in $\omega$ if there are at least $k$ distinct elements $t_1,\dots,t_k \in \Delta$ such that $\omega \models \psi(t_i)$. The other two counting quantifiers can be defined using: $(\exists_{\leq k} x : \psi(x)) \Leftrightarrow \neg (\exists_{\geq k+1} \psi(x))$ and $(\exists_{=k} x : \psi(x)) \Leftrightarrow (\exists_{\leq k} : \psi(x) \wedge \exists_{\geq k} : \psi(x))$.

\begin{example}
To give an example of the expressive power of FO$^2$ with counting quantifiers, we can notice that it is easy to constrain binary relations to be functions using it. In all models of the sentence $\forall x \exists_{=1} y : f(x,y)$, $f$ is a function from the domain to itself. Additionally, if we wanted to force $f$ to be a bijection, we could use $(\forall x \exists_{=1} y : f(x,y)) \wedge (\forall y \exists_{=1} x : f(x,y))$ etc. 
\end{example}



\subsection{Weighted First-Order Model Counting}\label{sec:wfomc}

In this section we formally describe {\em weighted first-order model counting}. We start by defining an auxiliary concept, {\em cardinality of a relation}.

\begin{definition}[Cardinality of Relation]
Let $\omega$ be a possible world and $R$ be a $k$-ary predicate. The cardinality of $R$ in $\omega$ is defined as 
$$N(R,\omega) = |\{R(x_1,\dots,x_k) \in \omega \}|,$$
i.e.\ $N(R,\omega)$ is the number of ground atoms of the predicate $R$ that are true in $\omega$.
\end{definition}

\begin{example}
Let $\omega = \{ \fr(\alice,\bob), \fr(\alice,\eve), \sm(\alice) \}$. Then 
$$N(\fr,\omega) = |\{ \fr(\alice,\bob), \fr(\alice,\eve) \}| = 2.$$
\end{example}

Next we define weighted first-order model counting.

\begin{definition}[WFOMC, \citeauthor{broeck2011completeness}, \citeyear{broeck2011completeness}]
Let $\Omega$ be a set of possible worlds over a given domain $\Delta$ ($\Omega$ will often be the set of all possible worlds over $\Delta$), $\mathcal{R}$ be the set of predicates in the language, $w(P)$ and $\overline{w}(P)$ be functions from predicates to complex\footnote{Normally, in the literature, the weights of predicates are real numbers. However, we will also use complex-valued weights in this paper, therefore we define the WFOMC problem accordingly using complex-valued weights.} numbers (we call $w$ and $\overline{w}$ {\em weight functions}). Then for a given first-order logic sentence $\Gamma$, we define
$$
    \wfomc(\Gamma,w,\overline{w},\Omega) = \sum_{\omega \in \Omega : \omega \models \Gamma} \prod_{R \in \mathcal{R}} w(R)^{N(R,\omega)} \cdot \overline{w}(R)^{|\Delta|^{\textit{arity}(R)}-N(R,\omega)}.
$$

\noindent We also define
$$\wfomc(\Gamma,w,\overline{w},\Delta) \stackrel{\textit{def}}{=} \wfomc(\Gamma,w,\overline{w}, \Omega_{\Delta}),$$
where $\Omega_\Delta$ is the set of all possible worlds on the domain $\Delta$ (using the predicates from $\mathcal{R}$).
\end{definition}

\noindent In this paper, when we do not explicitly define weights of some predicate $R$, we will assume that $w(R) = \overline{w}(R) = 1$.

Next we illustrate WFOMC on a small example. Additionally, we show how WFOMC can be used for inference in Markov logic networks in Section \ref{sec:mlns}.

\begin{example}
Let $\Delta = \{ A, B \}$, $\mathcal{R} = \{ \textit{heads}, \textit{tails} \}$, $w(\textit{heads}) = 2$, $w(\textit{tails}) = \overline{w}(\textit{heads}) = \overline{w}(\textit{tails}) = 1$, and $\Gamma = \forall x : (\textit{heads}(x) \vee \textit{tails}(x)) \wedge (\neg \textit{heads}(x) \vee \neg \textit{tails}(x))$. There are four models of $\Gamma$ on the domain $\Delta$: $\omega_1 = \{\textit{heads}(A), \textit{heads}(B) \}$, $\omega_2 = \{\textit{heads}(A), \textit{tails}(B) \}$, $\omega_3 = \{ \textit{tails}(A), \textit{heads}(A) \}$ and $\omega_4 = \{ \textit{tails}(A), \textit{tails}(B) \}$. The resulting weighted model count is $\wfomc(\Gamma,w,\overline{w},\Delta) = 4 + 2 + 2 + 1 = 9$.
\end{example}

\subsubsection{Two Useful Technical Results About WFOMC}

We now describe two useful technical properties of WFOMC. The first of these is about bit complexity of WFOMC. Later in the paper, we will need to be able to bound the bit complexity of WFOMC and the next proposition does exactly that (its proof is located in the appendix).

\begin{proposition}\label{prop:wfomcbits}
Let $\Gamma$ be a first-order logic sentence, $\mathcal{R} = \{R_1,R_2,\dots,R_m \}$ be the set of predicates from a given first-order language, $\Delta$ be a domain and $\Omega_\Delta$ be the set of all possible worlds on the domain $\Delta$ using the predicates from $\mathcal{R}$. Let $w$ and $\overline{w}$ be weight functions that assign to each predicate $R \in \mathcal{R}$ a rational number $w(R) = w'(R)/w''(R)$ and $\overline{w}(R) = \overline{w}'(R)/\overline{w}''(R)$, where $w'(R)$, $w''(R)$, $\overline{w}'(R)$ and $\overline{w}''(R)$ are integers. Let us further define $M = \max_{R} \max\{ |w'(R)|, |w''(R)|, |\overline{w}'(R)|, |\overline{w}''(R)| \}$. Then $\wfomc(\Gamma,w,\overline{w},\Delta)$ can be represented as a rational number $a/b$. The number of bits needed to encode the integers $a$ and $b$ is bounded by a polynomial in $|\Delta|$ and $\log M$ (but not in $m$!).
\end{proposition}

At some point in the paper, we will also need to replace certain subformulas by their negations, without actually using negation. This is possible using a technique described in \citep{beame2015symmetric}, stated in Appendix A.2 of their paper, which we restate in the proposition below.\footnote{The same transform also appears in \citep{meert2016relaxed} under the name ``relaxed Tseitin transform''.}

\begin{proposition}\label{prop:removing-negation}
Let $\neg \psi(x_1,\dots,x_k)$ be a subformula of a first-order logic sentence $\Phi$ with $k$ free variables $x_1$, $\dots$, $x_k$. Let $A$, $B$ be two new predicates of arity $k$. Let $\Phi'$ denote the sentence obtained from $\Phi$ by replacing the subformula $\neg \psi(x_1,\dots,x_k)$ with $A(x_1,\dots,x_k)$. Let 
\begin{multline*}
    \Upsilon = \forall x_1 \forall x_2 \dots \forall x_k : ((\psi(x_1,\dots,x_k) \vee A(x_1,\dots,x_k)) \\
    \wedge (A(x_1,\dots,x_k) \vee B(x_1,\dots,x_k)) \wedge (\psi(x_1,\dots,x_k) \vee B(x_1,\dots,x_k)))
\end{multline*}
and extend the given weight functions $w$ and $\overline{w}$ by defining $w(A) = \overline{w}(A) = w(B) = 1$ and $\overline{w}(B) = -1$. Then it holds 
$$\wfomc(\Phi, w, \overline{w}, \Omega) = \wfomc(\Phi' \wedge \Upsilon, w, \overline{w}, \Omega_\textit{ext})$$
where $\Omega$ is the set of all possible worlds on the domain $\Delta$ w.r.t.\ a given first-order logic language $\mathcal{L}$ and $\Omega_{\textit{ext}}$ is the set of all possible worlds on the domain $\Delta$ w.r.t.\ $\mathcal{L}$ extended by predicates $A$ and $B$.
\end{proposition}

\subsection{Domain-Lifted Inference}

Importantly, there are classes of first-order logic sentences for which weighted model counting can be solved in polynomial-time. In particular, let $\Omega$ be the set of all possible worlds over a given domain $\Delta$ and a given set of relations $\mathcal{R}$. As shown in \citep{van2014skolemization}, when the theory $\Gamma$ consists only of first-order logic sentences, each of which contains at most two logic variables, the weighted model count can be computed in time polynomial in the size of the domain $\Delta$. 
This is not the case in general when the number of variables in the formulas is greater than two unless P = \#P$_1$~\citep{beame2015symmetric}.\footnote{\#P$_1$ is the set of \#P problems over a unary alphabet.} Within statistical relational learning, the term used for problems that have such polynomial-time algorithms is {\em domain liftability}.

\begin{definition}[Domain liftability]
An algorithm for computing WFOMC with rational weights is said to be domain-liftable if it runs in time polynomial in the size of the domain and the number of bits needed to represent the weights.
\end{definition}

The definition of domain liftability presented here differs slightly from the original definition by \citet{broeck2011completeness} in that it also requires lifted algorithms to depend polynomially on the size of the representation of the formulas' weights. A justification for this definition follows from the work of \citet{jaeger} (Section 4.2). In particular, as pointed out by Jaeger, all existing domain-lifted exact-inference algorithms are also domain-lifted according to the definition that we use here. 

\subsubsection{An Application of WFOMC: Inference in Markov Logic Networks}\label{sec:mlns}

A Markov logic network \citep{Richardson2006} (MLN) is a set of weighted first-order logic formulas $(\alpha,w)$, where $w\in \mathbb{R}$ and $\alpha$ is a function-free first-order logic formula. The semantics are defined w.r.t.\ the groundings of the first-order logic formulas, relative to some finite set of constants $\Delta$, called the domain. An MLN $\Phi$ induces the probability distribution on possible worlds $\omega \in \Omega$ over a given domain:
\begin{equation}\label{eq:mln}
    P_{\Phi}(\omega) = \frac{1}{Z} \exp \left(\sum_{(\alpha,w) \in \Phi} w \cdot n(\alpha,\omega)\right),
\end{equation}
where $n(\alpha, \omega)$ is the number of groundings of $\alpha$ satisfied in $\omega$ (when $\alpha$ does not contain any variables, we define $n(\alpha,\omega) = \mathds{1}(\omega \models \alpha)$), and $Z$, called {\em partition function}, is a normalization constant to ensure that $p_{\Phi}$ is a probability distribution. We also allow infinite weights. A weighted formula of the form $(\alpha,+\infty)$ is understood as a hard constraint imposing that all worlds $\omega$ in which $n(\alpha,\omega)$ is not maximal have zero probability (this can also be deduced by taking the limit $w \rightarrow +\infty$). If all formulas in an MLN have at most $k$ variables, we call such an MLN {\em $k$-variable}.

Computation of the partition function $Z$ of an MLN can be converted to WFOMC. To compute the partition function $Z$ using weighted model counting, we proceed as \citep{van2011lifted}. Let an MLN $\Phi = \{(\alpha_1,w_1),\dots,(\alpha_m,w_m) \}$ over a set of possible worlds $\Omega$ be given.
For every $(\alpha_j,w_j) \in \Phi$, where the free variables in $\alpha_j$ are exactly $x_1$, $\dots$, $x_k$ and where $w \neq +\infty$, we create a new formula
$
    \forall x_1,\dots,x_k : \xi_j(x_1,\dots,x_k) \Leftrightarrow \alpha_j(x_1,\dots,x_k)
$
where $\xi_j$ is a new fresh predicate. When $w = +\infty$, we instead create a new formula $\forall x_1,\dots,x_k : \alpha_j(x_1,\dots,x_k)$. We denote the resulting set of new formulas $\Gamma$. Then we set
$w(\xi_j) = \exp{\left(w_j \right)}$
and $\overline{w}(\xi_j) = 1$ and for all other predicates we set both $w$ and $\overline{w}$ equal to 1. It is easy to check that then $\wfomc(\Gamma,w,\overline{w},\Omega) = Z$, which is what we needed to compute. To compute the marginal probability of a given first-order logic sentence $\gamma$, we have $\textit{P}_{\Phi}[X \models q] = \frac{\wfomc(\Gamma \cup \{ q \}, w, \overline{w},\Omega)}{\wfomc(\Gamma, w, \overline{w},\Omega)}$ where $X$ is sampled from the MLN.


For more examples of applications of weighted first-order model counting to statistical relational learning problems, we refer to \citep{van2013lifted}.




\section{Weighted Model-Counting Functions}\label{sec:modelcountingfunction}

Before getting to the {\em weighted model-counting functions}, we need to define notation for vectors of ``relation-cardinalities''. For a given possible world $\omega$ and a given list of predicates $\Psi = (R_1,R_2,...,R_m)$, we define the respective vector of relation-cardinalities as
$$\mathbf{N}(\Psi,\omega) \stackrel{def}{=} (n_1,\dots,n_m),$$
where $n_i = N(R_i,\omega)$ is the cardinality of the relation $R_i$ in $\omega$, i.e.\ the number of ground atoms of the form $R_i(c_1,\dots,c_{\operatorname{arity}(R_i)})$ that are true in $\omega$.

\begin{example}
Let $\omega = \{ \sm(\alice), \sm(\bob), \fr(\alice,\bob) \}$ and $\Psi = (\sm, \fr)$. Then the vector of relation-cardinalities is $\mathbf{N}(\Psi,\omega) = (2,1)$.
\end{example}

Next we define {\em model-counting function} (which we will also call {\em MC-function}).

\begin{definition}[Model-Counting Function]
Let $\Omega$ be a set of possible worlds and let $\Psi = (R_1,R_2, \dots,R_m)$ be a list of predicates. We define the model counting function as 
$$\operatorname{MC}_{\Psi,\Omega}(\mathbf{n}) = |\{ \omega \in \Omega | \mathbf{N}(\Psi,\omega) = \mathbf{n} \}|.$$
Given first-order logic sentence $\Gamma$ and a domain $\Delta$ we also define 
$$\operatorname{MC}_{\Psi,\Gamma,\Delta}(\mathbf{n}) \stackrel{\textit{def}}{=} \operatorname{MC}_{\Psi,\Omega_{\Gamma,\Delta}}(\mathbf{n}),$$ 
where $\Omega_{\Gamma,\Delta}$ is the set of models of the sentence $\Gamma$ on the domain $\Delta$ (assuming some given first-order language that also specifies the predicates).
\end{definition}

\noindent Intuitively, for any $\mathbf{n} \in \mathbb{Z}^m$, the model counting function gives us the number of possible worlds (from the given set $\Omega$) that satisfy $\mathbf{N}(\Psi,\omega) = \mathbf{n}$.

\begin{figure}
    \centering
    \includegraphics[scale=0.5]{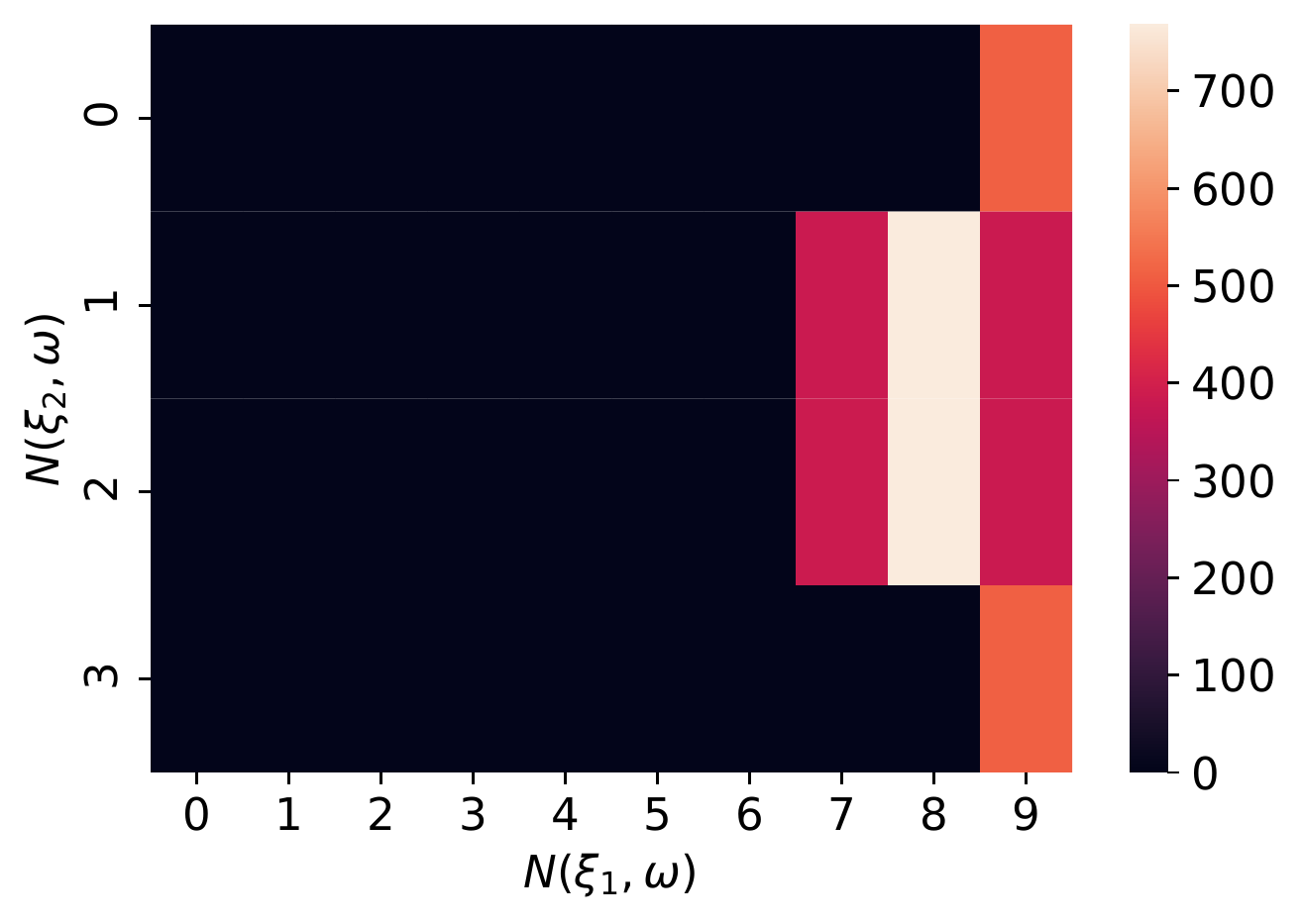}
    \includegraphics[scale=0.5]{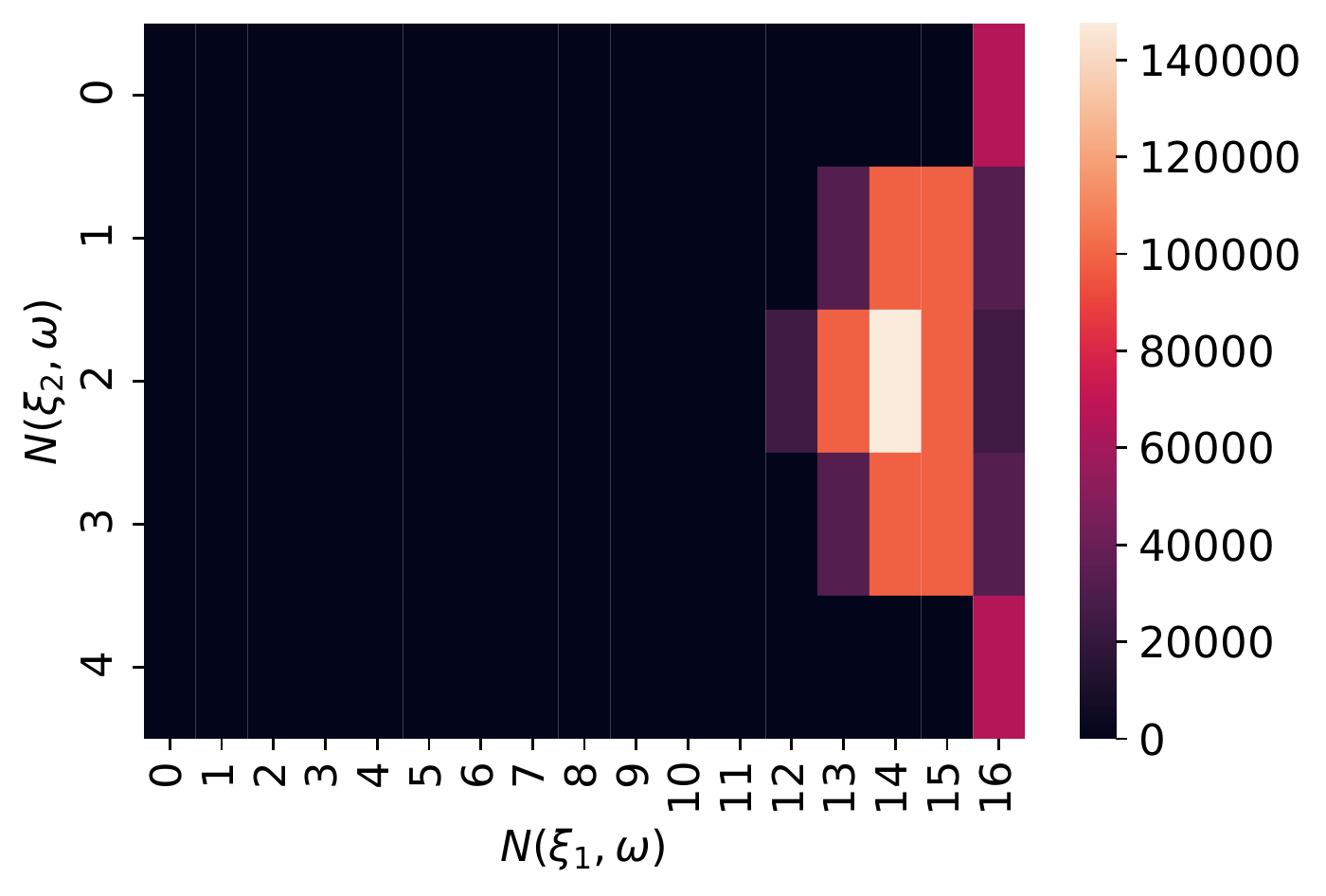}
    \caption{Model-counting functions from Example \ref{example:mc}. 
    }
    \label{fig:mc}
\end{figure}

\begin{example}\label{example:mc0}
Let us consider the domain $\Delta = \{A,B,C,D \}$, the sentence $\Gamma = \forall x : (\textit{heads}(x) \vee \textit{tails}(x)) \wedge (\neg \textit{heads}(x) \vee \neg \textit{tails}(x))$ and the list $\Psi = (\textit{heads})$. We assume that the first-order language over which the possible worlds are defined contains only the predicates $\textit{heads}$ and $\textit{tails}$. Then we have, for instance, $\mathbf{N}_{\Psi,\Gamma,\Delta}(0,0) = 0$ (since the sentence $\Gamma$ makes it impossible for any domain element to be neither {\em heads} nor {\em tails}) and $\mathbf{N}_{\Psi,\Gamma,\Delta}(1,3) = 4$ and so on.
\end{example}

\begin{example}\label{example:mc}
In Figure \ref{fig:mc}, we show examples of two MC-functions, $\operatorname{MC}_{\Psi,\Gamma,\Delta}(\mathbf{n})$, where $\Psi = (\xi_1,\xi_2)$, $\Gamma = (\forall x \forall y : \xi_1(x,y) \Leftrightarrow (\sm(x) \wedge \fr(x,y) \Rightarrow \sm(y))) \wedge (\forall x : \xi_2(x) \Leftrightarrow \sm(x))$, and $\Delta$ is a domain of size $3$ and $4$, respectively. Note that the form of the sentence $\Gamma$ corresponds to the encoding of an MLN with two formulas $\alpha = \sm(x) \wedge \fr(x,y) \Rightarrow \sm(y)$ and $\beta = \sm(x)$ (cf. Section \ref{sec:mlns}).
\end{example}

The concept of model-counting function can be straightforwardly generalized to weighted model-counting functions that we define next. Weighted model-counting functions are the main ``work-horses'' that we use in the rest of the paper.

\begin{definition}[Weighted Model-Counting Function]
Let $\Omega$ be a set of possible worlds and let $\Psi = (R_1,R_2, \dots,R_m)$ be a list of predicates. We define the model counting function as: 
$$\operatorname{WMC}_{\Psi,\Omega}(\mathbf{n}, w, \overline{w}) = \wfomc(\top,w,\overline{w},\{ \omega \in \Omega | \mathbf{N}(\Psi,\omega) = \mathbf{n} \}),$$
where $\top$ (tautology) is the trivial sentence which is always true. Given a first-order logic sentence $\Gamma$ and a domain $\Delta$, we also define 
$$\operatorname{WMC}_{\Psi,\Gamma,\Delta}(\mathbf{n}, w, \overline{w}) \stackrel{\textit{def}}{=} \operatorname{WMC}_{\Psi,\Omega_{\Gamma,\Delta}}(\mathbf{n}, w, \overline{w}),$$ 
where $\Omega_{\Gamma,\Delta}$ is the set of models of the sentence $\Gamma$ on the domain $\Delta$ (assuming some given first-order language that specifies the predicates).
\end{definition}

\noindent Clearly, model-counting functions are a special case of weighted model counting functions for $w \equiv 1$ and $\overline{w} \equiv 1$.

In the next subsection we explain how to compute weighted model-counting functions using a WFOMC oracle.


\subsection{Computing Weighted Model-Counting Functions}

At first it may not be obvious how to compute weighted model-counting functions efficiently. In \citep{kuzelka.complex}, we described a method based on discrete Fourier transform that can be used for computing weighted model-counting functions (we describe this method in the appendix). A downside of this method is that it requires computing WFOMC over complex numbers. Even though existing lifted inference algorithms, which normally only count over real-valued weights, can be straightforwardly modified to also allow complex numbers, it would be nicer if we could do the same without modifying them. In this section, we show that it is possible to compute weighted model-counting functions efficiently also without WFOMC over complex numbers. In particular, we prove the following proposition.

\begin{proposition}\label{prop:dft}
Let $\Delta$ be a set of domain elements, $\Gamma$ be a first-order logic sentence and $\Psi = (R_1, \dots, R_m)$ be a list of relations.
If $\wfomc(\Gamma,w,\overline{w},\Delta)$ can be computed in time polynomial in $|\Delta|$ and in the number of bits needed to encode $w$ and $\overline{w}$ then the corresponding WMC-function  $\operatorname{WMC}_{\Psi,\Gamma,\Delta}(\mathbf{n},w,\overline{w})$ can also be computed in time polynomial in $|\Delta|$ and in the number of bits needed to encode $w$ and $\overline{w}$.
\end{proposition}



\begin{proof}

We prove this proposition by showing how to compute weighted model-counting functions using Lagrange interpolation when we have access to a WFOMC oracle.

Let $w^*$ and $\overline{w}^*$ be weight functions defined by: $w^*(R) = \overline{w}^*(R) = 1$ for all $R \in \Psi$ and $w^*(R) = w(R)$ and $\overline{w}^*(R) = \overline{w}(R)$ for all the other predicates $R \not\in \Psi$. Then we can write: 
\begin{equation*}
    \wfomc(\Gamma,w,\overline{w}^*,\Delta) = \sum_{\mathbf{n} \in \mathcal{D}} \operatorname{WMC}_{\Psi,\Gamma,\Delta}(\mathbf{n}, w^*, \overline{w}^*) \cdot \prod_{i=1}^m w(R_i)^{n_i}
\end{equation*}
where $n_i$ denotes the $i$-th component of $\mathbf{n}$, 
$$\mathcal{D} = \left\{0,1,\dots, M_1 \right\} \times \left\{0,1,\dots,M_2\right\}\times \dots \times \left\{ 0, 1, \dots, M_m \right\},$$
and $M_1 = |\Delta|^{\textit{arity}(R_1)}$, $M_2 = |\Delta|^{\textit{arity}(R_2)}$, $\dots$, $M_m = |\Delta|^{\textit{arity}(R_m)}$.


It follows that, when we fix all weights of all predicates $R \not\in \Psi$ (i.e. if we keep them constant), $\wfomc(\Gamma,w,\overline{w}^*,\Delta)$ becomes a polynomial in the weights $w(R_1), \dots, w(R_m)$. Let us denote this polynomial as $W(w_1,\dots,w_m)$. Importantly, $\operatorname{WMC}_{\Psi,\Gamma,\Delta}(\mathbf{n}, w^*, \overline{w}^*)$ is the coefficient of $\prod_{i=1}^m w(R_i)^{n_i}$ in this polynomial, where $n_i$ denotes the $i$-th component of $\mathbf{n}$. 
If we can extract the coefficients of the monomials efficiently, it will mean that we can efficiently compute the WMC-function using an oracle for WFOMC. For that we first introduce another, univariate, polynomial: 

\begin{equation*}
    W_0(t) \stackrel{def}{=} W\left(t, t^{M_1+1}, t^{(M_1 + 1)(M_2 + 1)}, \dots, t^{(\dots (M_1 + 1)(M_2 + 1)) \dots) (M_{m-1}+1)}\right).
\end{equation*}

\noindent The polynomial $W_0$ has the convenient property that the coefficient of 
$$t^{n_1 + (M_1+1) n_2 + \dots + (\dots (M_1 + 1)(M_2 + 1)) \dots) n_m}$$ 
is equal to the coefficient of 
$w_1^{n_1} w_2^{n_2} \dots w_m^{n_m}$ 
in $W(w_1,w_2,\dots,w_m)$. Let us denote this coefficient by $A_{n_1,\dots,n_m}$. It follows that $A_{n_1,\dots,n_m}$ is also equal to $\operatorname{WMC}_{\Psi,\Gamma,\Delta}(\mathbf{n}, w^*, \overline{w}^*)$, from which we can then obtain 

\begin{multline*}
    \operatorname{WMC}_{\Psi,\Gamma,\Delta}(\mathbf{n}, w, \overline{w})  = \operatorname{WMC}_{\Psi,\Gamma,\Delta}(\mathbf{n}, w^*, \overline{w}^*) \cdot \prod_{i=1}^m w(R_i)^{n_i} \cdot \overline{w}(R)^{|\Delta|^{\textit{arity}(R_i)}-N(R_i,\omega)} \\
    = A_{n_1,\dots,n_m} \cdot \prod_{i=1}^m w(R_i)^{n_i} \cdot \overline{w}(R)^{|\Delta|^{\textit{arity}(R_i)}-N(R_i,\omega)}.
\end{multline*}

\noindent The only missing part is to show that we can extract the coefficient $A_{n_1,\dots,n_m}$ efficiently using a WFOMC oracle. This can be done using Lagrange interpolation (cf Section \ref{sec:lagrange}). First, we define $|\mathcal{D}|+1$ points $(x_i,y_i)$ for the polynomial interpolation problem: $(1, W_0(1))$, $(2, W_0(2))$, $\dots$, $(|\mathcal{D}|+1, W_0(|\mathcal{D}|+1))$. Using Proposition \ref{prop:wfomcbits}, we have that the number of bits needed to encode each of $W_0(1)$, $W_0(2)$, $\dots$, $W_0(|\mathcal{D}|+1)$ is polynomial in $|\mathcal{D}|$. Combining that with Proposition~\ref{prop:lagrange}, we then have that the number of bits needed to encode the coefficients of the polynomial interpolating the points $(1, W_0(1))$, $(2, W_0(2))$, $\dots$, $(|\mathcal{D}|+1, W_0(|\mathcal{D}|+1))$ grows only polynomially with $|\mathcal{D}|$. Since $|\mathcal{D}|$ is itself bounded by a polynomial in the size of the domain~$|\Delta|$, it follows that we can extract the coefficients and consequently the WMC-function that we want to compute in time polynomial in~$|\Delta|$ and the number of bits needed to encode $w$ and $\overline{w}$.

\end{proof}

\begin{remark}
We could do a bit better in terms of practical efficiency than the construction from the above proof if we replaced the univariate Lagrange interpolation by its multivariate version (we could use, e.g., Lemma 5 from \citep{koiran2011interpolation}). Then we would only need to evaluate WFOMC on weights from the set $\{0, 1,2,\dots,|\mathcal{D}|\}$. For simplicity, in the proof of the above proposition we opted for the more elementary approach, which is enough for our purposes.
\end{remark}

\section{WFOMC with Cardinality Constraints}

In this paper, a {\em simple cardinality constraint} is en expression of the form $|R| \in \mathcal{A}$ where $R$ is a predicate and $\mathcal{A} \subseteq \mathbb{N}$. A possible world $\omega$ satisfies a given cardinality constraint $|R| = k$ if $N(R,\omega) \in \mathcal{A}$, i.e.\ if the number of ground atoms $R_i(t_1,\dots,t_n)$ that are true in $\omega$ is in $\mathcal{A}$. We write $\omega \models (|R| \in \mathcal{A})$ when the cardinality constraint $|R| \in \mathcal{A}$ is satisfied in $\omega$. We will also use the notation $|R| \bowtie k$, where $\bowtie \in \{=,\leq,\geq,<,> \}$ and $k \in \mathbb{N}$. So, e.g., $|R| \leq k$ is a short for $|R| \in \{0,1,\dots,k\}$. Finally, we allow cardinality constraints as atomic formulas in first-order logic formulas. For instance, $(|f| = 2) \wedge (\forall x \forall y : f(x,y) \Rightarrow f(y,x))$ is a valid formula (its models can be interpreted as undirected graphs with exactly one edge) and the satisfaction relation $\models$ is extended naturally.

Computing WFOMC with cardinality constraints can be done using WMC-functions. Moreover as the next proposition shows, domain-liftability is preserved when we add cardinality constraints to a sentence which is domain-liftable.

\begin{proposition}\label{prop:cardinality}
Let $\Gamma$ be a first-order logic sentence, $\psi(x_1,\dots,x_m)$ be propositional logic formula and let $\Upsilon = \psi(|R_{i_1}| \bowtie k_1, \dots, |R_{i_m}| \bowtie k_m)$, where $\psi$ is a Boolean formula and $\bowtie \in \{=,\leq,\geq,<,> \}$. If computing the WFOMC of $\Gamma$ is domain-liftable then computing the WFOMC of $\Gamma \wedge \Upsilon$ is also domain-liftable.
\end{proposition}
\begin{proof}
Let $\Psi = (R_{i_1},\dots,R_{i_m})$. Since computing $\wfomc(\Gamma,w,\overline{w},\Delta)$ is domain-liftable, so is computing $\operatorname{WMC}_{\Psi,\Gamma,\Delta}(\mathbf{n})$, which follows from Proposition \ref{prop:dft}. In addition we only need to evaluate the WMC-function on a set of polynomially-many (in $|\Delta|$) points, specifically on the set $\mathcal{D} = \{0,1,\dots, M_1 \} \times \{0,1,\dots, M_2\}\times \dots \times \{ 0, 1, \dots, M_m \}$ 
where $M_1 = |\Delta|^{\textit{arity}(R_{i_1})},$ $M_2 = |\Delta|^{\textit{arity}(R_{i_2})}$, $\dots$, $M_m = |\Delta|^{\textit{arity}(R_{i_m})}.$ Finally, we can compute $\wfomc(\Gamma \wedge \Upsilon,w,\overline{w},\Delta)$ as:
\begin{equation*}
    \wfomc(\Gamma,w,\overline{w},\Delta) = \sum_{\mathbf{n} \in \mathcal{D}} \psi(|R_{i_1}| \bowtie k_1, \dots, |R_{i_m}| \bowtie k_m) \cdot  \operatorname{WMC}_{\Psi,\Gamma,\Delta}(\mathbf{n})
\end{equation*}
where $n_i$ denotes the $i$-th component of the vector $\mathbf{n}$. Hence, $\wfomc(\Gamma \wedge \Upsilon,w,\overline{w},\Delta)$ can be computed in time polynomial in the size of the domain which finishes the proof.
\end{proof}


\begin{remark}
The techniques from this section do not apply only to encoding of cardinality constraints. It is easy to replace the Boolean formula $\psi$ by a function to rational numbers and show that one can efficiently compute weighted first-order model counts with ``non-multiplicative'' weight functions, i.e.\ with weight functions that only depend on $\mathbf{N}(\Psi, \omega)$. We will not need this more general setting in this paper.
\end{remark}

\section{WFOMC with Functionality Constraints\protect\footnote{In the next section, we generalize the results presented here to allow constraints of the form $\forall x \exists_{=k} y : \psi(x,y)$, of which $\forall x \exists_{=1} y : \psi(x,y)$ is a special case. However, the case of functionality constraints, besides being easier to understand, is important on its own and has been studied in the literature \citep{DBLP:conf/lics/KuusistoL18}. We note that whereas \citet{DBLP:conf/lics/KuusistoL18} only allowed one functionality constraint, here we already allow multiple functionality constraints.}}

A {\em functionality constraint} is a constraint expressed by a first-order-logic sentence of the form 
$$\forall x \exists_{=1} y : \psi(x,y),$$ 
which asserts that for every $x$ there is exactly one $y$ such that $\psi(x,y)$ is true. In this section we show how to compute WFOMC of a 2-variable first-order logic sentence with an arbitrary number of functionality constraints while still guaranteeing runtime polynomial in the domain size $|\Delta|$. 

First, we can notice that we can replace any functional constraint of the form $\forall x \exists_{=1} y : \psi(x,y)$, where $\psi(x,y)$ is a formula, with free variables exactly $x$ and $y$, by $(\forall x \forall y : \xi(x,y) \Leftrightarrow \psi(x,y) ) \wedge (\forall x \exists_{=1} y : \xi(x,y))$, where $\xi$ is a fresh predicate not occurring anywhere else. Therefore we will assume without loss of generality that the only functional constraints are of the form $\forall x \exists_{=1} y : R(x,y)$ where $R$ is a predicate. The main result of this section is then the following theorem.

\begin{theorem}\label{prop:function}
Let $\Gamma$ be an FO$^2$ sentence and $\Upsilon = (|R_{i_1}| \bowtie k_1) \wedge \dots \wedge (|R_{i_m}| \bowtie k_m) \wedge (\forall x \exists_{=1} y : R_{i_1}(x,y)) \wedge \dots \wedge (\forall x \exists_{=1} y : R_{i_{m'}}(x,y))$ be a conjunction of cardinality and functionality constraints where $\bowtie \in \{=,\leq,\geq,<,> \}$. Computing the WFOMC of $\Gamma \wedge \Upsilon$ is domain-liftable.
\end{theorem}

Next we prove a simple lemma that will allow us to reduce WFOMC with functionality (and possibly also cardinality) constraints to WFOMC involving only cardinality constraints and no functionality constraints.

\begin{lemma}\label{lemma:lemma1}
Let $\Omega$ be the set of all possible worlds on a domain $\Delta$. Let $\Gamma$ be a first-order logic sentence. Let
$\Phi = (\forall x \exists_{=1} y : R_{i_1}(x,y)) \wedge \dots \wedge (\forall x \exists_{=1} y : R_{i_h}(x,y))$
and
$
    \Phi' = (\forall x \exists y: R_{i_1}(x,y)) \wedge (|R_{i_1}| = |\Delta|) \wedge
    \dots \wedge (\forall x \exists y: R_{i_h}(x,y)) \wedge (|R_{i_h}| = |\Delta|).
$
Then for all $\omega \in \Omega$: $(\omega \models \Gamma \wedge \Phi) \Leftrightarrow (\omega \models \Gamma \wedge \Phi')$.
\end{lemma}
\begin{proof}
Let $R$ be any of the relations $R_{i_1}$, $\dots$, $R_{i_h}$. 
The constraint $\forall x \exists_{=1} y : R(x,y)$ can be rewritten as: (i) $\forall x \exists y : R(x,y)$ and (ii) $\forall x,y,z: R(x,y) \wedge R(x,z) \Rightarrow y = z$.
($\Rightarrow$) It follows from (i) that $|R| \geq |\Delta|$. If $|R| > |\Delta|$ then by the pigeon-hole principle, there must be at least one $s \in \Delta$ such that $R(s,t)$ and $R(s,t')$ for some $t \neq t' \in \Delta$ which contradicts (ii). Hence, $\forall x \exists_{=1} y : R(x,y)$ implies $|R| = |\Delta|$ and $\forall x \exists y : R(x,y)$.
($\Leftarrow$) What we need to show is that if $(\forall x \exists y: R(x,y)) \wedge (|R| = |\Delta|)$ holds then (i) and (ii) must hold as well. Clearly, (i) must hold. So let us suppose, for contradiction, that $(\forall x \exists y: R(x,y)) \wedge (|R| = |\Delta|)$ holds but there is some $s \in \Delta$ such that $R(s,t)$ and $R(s,t')$ for some $t \neq t' \in \Delta$. We have $|\{ (x,y) \in \Delta^2 | R(x,y) \wedge x \neq s \}| \geq |\Delta|-1$ (from $\forall x \exists y : R(x,y)$). Therefore it is easy to see that $|R| \geq |\{ (x,y) \in \Delta^2 | R(x,y) \wedge x \neq s \}| + 2 > |\Delta|$, which is a contradiction.
\end{proof}

\noindent Note that the constraints $|R_{i_1}| = |\Delta|$, $\dots$, $|R_{i_h}| = |\Delta|$ are cardinality constraints which we already know how to deal with.

We are now ready to prove Theorem \ref{prop:function}.

\begin{proof}[Proof of Theorem \ref{prop:function}]
Using Lemma \ref{lemma:lemma1}, we can reduce the problem of computing the WFOMC of $\Gamma \wedge \Upsilon$ to the problem of computing the WFOMC of $\Gamma \wedge \Upsilon' \wedge \Upsilon''$ where $\Upsilon$ contains only cardinality constraints and $\Upsilon'$ has the form $(\forall x \exists y : \psi_1(x,y)) \wedge \dots \wedge (\forall x \exists y : \psi_{m'}(x,y))$. Since $\Gamma \wedge \Upsilon''$ is an FO$^2$ sentence (hence domain-liftable) and $\Upsilon'$ is a conjunction of cardinality constraints, we can use Proposition \ref{prop:cardinality} to finish the proof. 
\end{proof}


Next we provide an illustration of the techniques derived so far. 

\begin{figure}
\centering
\includegraphics[width=0.49\linewidth]{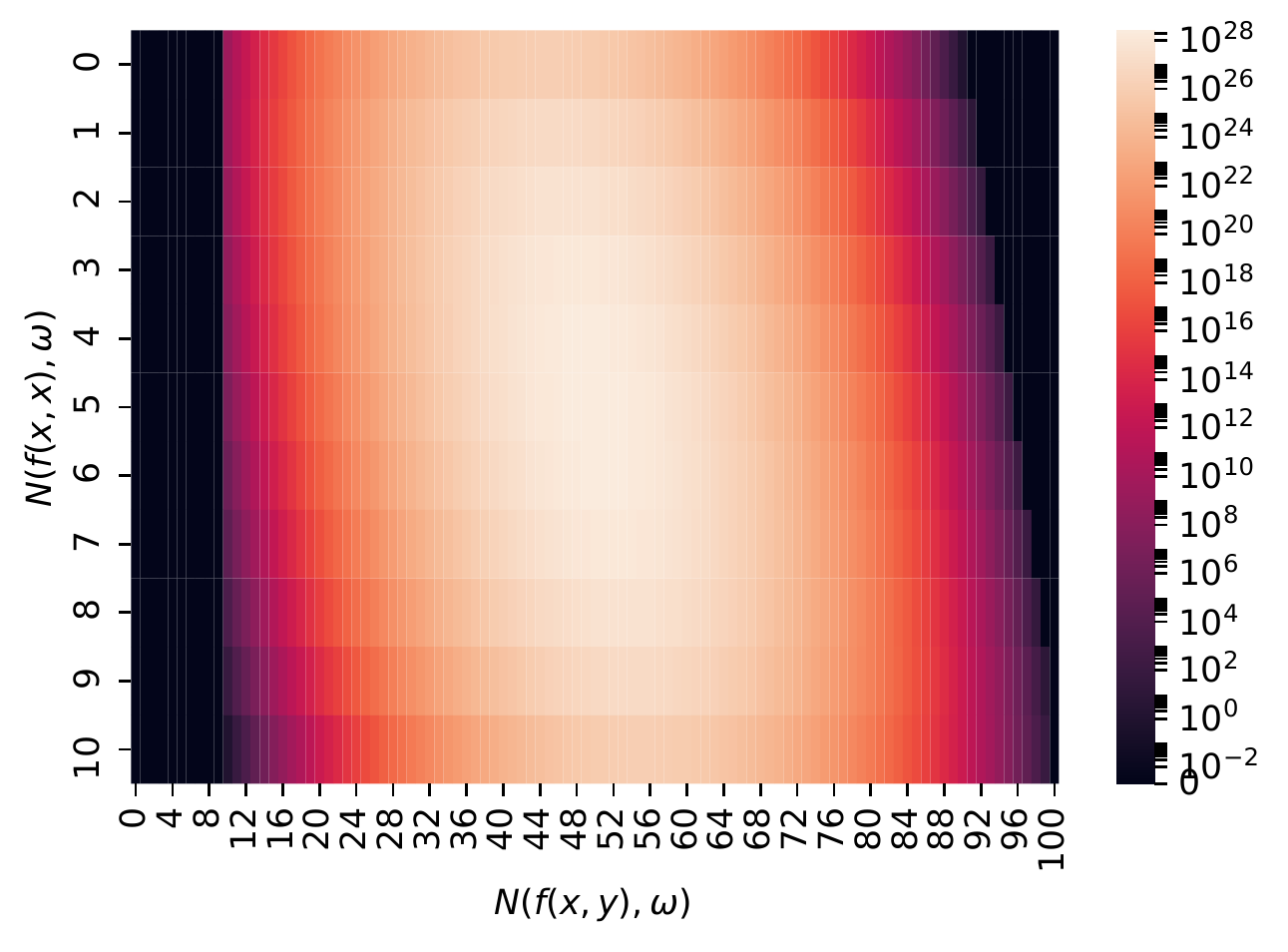}
\includegraphics[width=0.49\linewidth]{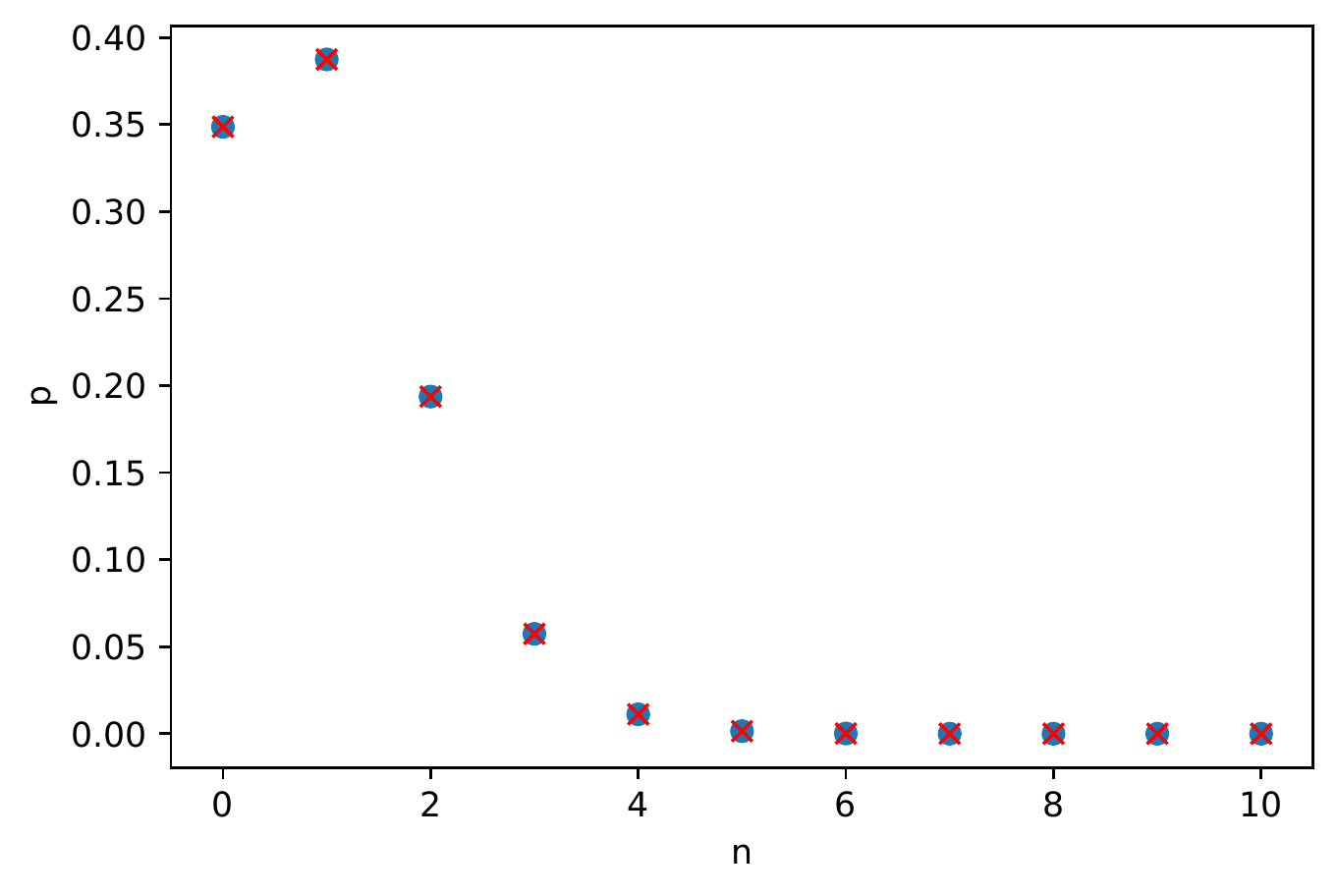}
\caption{{\bf Left:} The MC-function from Example \ref{example:functions}. {\bf Right: } The distribution of the number of fixed points (see Example \ref{example:functions}).}\label{figure}
\end{figure}

\begin{example}\label{example:functions}
How many fixed points does a uniformly sampled function from $\{1,2,\dots,n \}$ to itself have? We can answer this question using WFOMC with functionality constraints. We define 
$$\Gamma =  (\forall x \exists y : f(x,y)) \wedge (\forall x : \xi(x) \Leftrightarrow f(x,x)).$$ 
Here, we introduced a fresh new predicate $\xi$ such that, for all $t \in \Delta$, $\xi(t)$ is true if and only if $t$ is a fixed point of $f$. Next we define 
$$\Psi = \{f,\xi \}.$$ 
Then, using the techniques described in Section \ref{sec:modelcountingfunction}, we compute the MC-function $\operatorname{MC}_{\Psi,\Gamma,\Delta}$, which is shown for $n = 10$ in the left panel of Figure \ref{figure}. 
Now we can extract the distribution that we wanted to compute from this MC-function by ignoring all points except those where $|f| = |\Delta|$. That is, the probability $p_k$ that a uniformly sampled function has $k$ fixed points is equal to: 
$$p_k = \frac{\operatorname{MC}_{\Psi,\Gamma,\Delta}(|\Delta|,k)}{\sum_{j = 1}^{|\Delta|} \operatorname{MC}_{\Psi,\Gamma,\Delta}(|\Delta|,j)}.$$ 
We show the computed distribution in the right panel of Figure \ref{figure} (blue circles). As a sanity check, we also computed the distribution analytically using the formula $\binom{n}{k} (n-1) ^{n-k}/n^n$ and displayed it in the same plot (red crosses). As expected, the values computed using the two approaches are the same.
\end{example}

We give a slightly more complex example, computing the number of anti-involutive functions, in Appendix~\ref{appendix:anti-involutive}.




\section{The $\exists_{=k}$-Quantifier}\label{sec:exists-k-quantifier-section}

In this section we show how the techniques used for WFOMC with functionality constraints, described in the previous section, can be further generalized to more complex settings. In particular, we show how to use them to compute WFOMC with the counting quantifier $\exists_{=k}$.

\subsection{Two Types of Constraints: $\forall x \exists_{=k} y$ and $\exists_{=k} x \forall y$}\label{sec:exists-constraints}

We start by showing how to extend domain-liftability to FO$^2$ with constraints of the form $\forall x \exists_{=k} y : R(x,y)$ and $\exists_{=k} x \forall y : R(x,y)$, where $\exists_{=k}$ is the counting quantifier ``exists exactly k''. 

\begin{theorem}\label{prop:existsk}
Let $\Gamma$ be an FO$^2$ sentence and 
\begin{multline*}
    \Upsilon = (|R_{i_1}| \bowtie c_1) \wedge \dots \wedge (|R_{i_m}| \bowtie c_m) \\
    \wedge (\forall x \exists_{=k_1} y : \psi_1(x,y)) \wedge \dots \wedge (\forall x \exists_{=k_{m'}} y : \psi_{m'}(x,y)) \\
    \wedge (\exists_{=k_1'} x \forall y : \psi_1(x,y)) \wedge \dots \wedge (\exists_{=k_{m''}'} x \forall y : \psi_{m'}(x,y)), 
\end{multline*}
where $\bowtie \in \{=,\leq,\geq,<,> \}$.
Then computing the WFOMC of $\Gamma \wedge \Upsilon$ is domain-liftable.
\end{theorem}

To prove Theorem \ref{prop:existsk}, we will need the following two lemmas.

\begin{lemma}\label{lemma:lemma2}
Let $\Omega$ be the set of all possible worlds on a domain $\Delta$. Let $\Gamma$ be a first-order logic sentence. Let
$\Phi$ be a first-order logic sentence with cardinality constraints, defined as follows:
\begin{multline*}
\Phi = (|R| = k \cdot |\Delta|) \wedge (\forall x,y : R(x,y) \Leftrightarrow (f_1^R(x,y) \vee f_2^R(x,y) \vee \dots \vee f_k^R(x,y))) \\
\wedge \bigwedge_{i=1}^k (\forall x \exists y : f_i^R(x,y)) \wedge \bigwedge_{i,j=1, i \neq j}^k (\forall x,y : \neg f_i^R(x,y) \vee \neg f_j^R(x,y)).
\end{multline*}
Then it holds:
$$\wfomc(\Gamma \wedge \forall x \exists_{=k} y : R(x,y), w, \overline{w}, \Omega) =  \frac{1}{(k!)^{|\Delta|}}\wfomc(\Gamma \wedge \Phi, w, \overline{w}, \Omega_{\textit{ext}}),$$
where $\Omega$ is the set of all possible worlds on the domain $\Delta$ w.r.t.\ a given first-order logic language $\mathcal{L}$ and $\Omega_{\textit{ext}}$ is the set of all possible worlds on the domain $\Delta$ w.r.t.\ $\mathcal{L}$ extended by predicates $f_1^{R}$, $\dots$, $f_k^R$.
\end{lemma}
\begin{proof}
First we show that, for all $\omega \in \Omega_{\textit{ext}}$, it holds: 
if $\omega \models \Gamma \wedge \Phi$ then $\omega \models \Gamma \wedge \forall x \exists_{=k} y : R(x,y)$. 
The sentence $\Phi$ implies that for every $t_1,t_2 \in \Delta$ such that $R(t_1,t_2)$ is true, there is exactly one $i\in \{1,2,\dots,k\}$ such that $f_i^R(t_1,t_2)$ is true. It follows that $|R| = |f_1^R|+\dots+|f_k^R|$. Now, using similar reasoning as in the proof of Lemma \ref{lemma:lemma1}, we can see that $|R| = |f_1^R|+\dots+|f_k^R| = k \cdot |\Delta|$ together with $\bigwedge_{i=1}^k (\forall x \exists y : f_i^R(x,y))$ and $\bigwedge_{i,j=1, i \neq j}^k (\forall x,y : \neg f_i^R(x,y) \vee \neg f_j^R(x,y))$ also implies that all of $f_1^R(x,y)$, $\dots$, $f_k^R(x,y)$ must be functions. 
It follows that $\forall x \exists_{=k} y : R(x,y)$ must be true in any possible world $\omega \in \Omega$ that satisfies $\Phi$.

To finish the proof, let $[\omega]_\mathcal{L}$ denote the ``projection'' of $\omega$ on the language $\mathcal{L}$ which is the possible world obtained from $\omega$ by removing all atoms whose predicates are not contained in $\mathcal{L}$ (i.e.\ $f_1^R$, $\dots$, $f_k^R$). One can show easily that, for every model $\omega \in \Omega$ of the sentence $\Gamma \wedge \forall x \exists_{=k} y : R(x,y)$, there are exactly $(k!)^{|\Delta|}$ models $\omega' \in \Omega_{\textit{ext}}$ such that $\omega = [\omega']_\mathcal{L}$, which follows from the following: (i) if, for any $t \in \Delta$, we permute $t_1$, $t_2$, $\dots$, $t_k$ in $f_1^R(t,t_1)$, $f_2^R(t,t_2)$ $\dots$, $f_k^R(t,t_k)$ in the model $\omega'$, we get another model of $\Gamma \wedge \Phi$, (ii) up to these permutations, the predicates $f_i^k$ in $\omega'$ are determined uniquely by $\omega$. Finally, the weights of all these $\omega'$s are the same as those of $\omega$.
\end{proof}

\begin{lemma}\label{lemma:lemma3}
Let $\Omega$ be the set of all possible worlds on a domain $\Delta$. Let $\Gamma$ be a first-order logic sentence and $U$ and $R$ be predicates. 
Then it holds: 
\begin{multline*}
    \wfomc(\Gamma \wedge \exists_{=k} x \forall y : R(x,y), w, \overline{w}, \Omega) \\ =  \wfomc(\Gamma \wedge (\exists_{=k} x : U^R(x)) 
    \wedge (\forall x : U^R(x) \Leftrightarrow (\forall y : R(x,y))), w, \overline{w}, \Omega_\textit{ext}).
\end{multline*}
where $\Omega$ is the set of all possible worlds on the domain $\Delta$ w.r.t.\ a given first-order logic language $\mathcal{L}$ and $\Omega_{\textit{ext}}$ is the set of all possible worlds on the domain $\Delta$ w.r.t.\ $\mathcal{L}$ extended by the predicate $U^R$ (in particular, we assume w.l.o.g.\ that $\mathcal{L}$ did not originally contain this predicate).
\end{lemma}
\begin{proof}
The proof is straightforward.
\end{proof}

We are now ready to prove Theorem \ref{prop:existsk}.

\begin{proof}[Proof of Theorem \ref{prop:existsk}]
First, repeatedly using Lemma \ref{lemma:lemma3}, we can get rid of all constraints of the form $\exists_{=k} x \forall y : R(x,y)$. Besides new first-order logic sentences, this also produces new constraints of the form $\exists_{=k} x : U^R(x)$ which can be easily encoded using cardinality constraints $|U^R| = k$. Finally, we can use Lemma \ref{lemma:lemma2} repeatedly to get rid of the constraints of the form $\forall x \exists_{=k} y : R(x,y)$. Since the resulting sentence contains only two variables and cardinality constraints, it follows from Proposition \ref{prop:cardinality} that we can compute its WFOMC in time polynomial in the size of the domain~$\Delta$.
\end{proof}


\subsubsection{An Illustration: Counting K-Regular Graphs}\label{sec:regular-graphs}

We now illustrate the techniques developed in this section on the problem of computing the number of 2-regular graphs on $n$ vertices. An undirected graph is called {\em k-regular} if all its vertices have degree $k$. Note that we do not count non-isomorphic graphs here.

We start by writing down the axioms defining $2$-regular graphs: 

\begin{align}
    \forall x &: \neg e(x,x), \label{fo:noloops} \\
    \forall x,y &: e(x,y) \Rightarrow e(y,x), \label{fo:symmetry} \\
    \forall x \exists_{=2} y &: e(x,y). \label{fo:degree2}
\end{align}

\noindent Here (\ref{fo:noloops}) forbids self-loops, (\ref{fo:symmetry}) requires $e$ to be a symmetric relation (to model undirected graphs) and (\ref{fo:degree2}) requires every vertex to have two out-going edges. Since, by (\ref{fo:symmetry}), edges are guaranteed to be symmetric, (\ref{fo:degree2}) is enough to guarantee that every vertex will have degree 2.

Since every sentence in the above theory has at most 2 variables and contains only quantifiers $\forall$ and $\exists_{=k}$, we can apply the techniques developed in this section to compute the number of 2-regular graphs using WFOMC. We now provide details. First, we can rewrite (\ref{fo:degree2}) using only $\forall$, $\exists$ and cardinality constraints as in Lemma \ref{lemma:lemma2}:

\begin{align}
    \forall x \forall y &: \xi(x,y) \Leftrightarrow e(x,y),\label{fo:degree2-a} \\
    & |\xi| = 2 |\Delta|, \label{fo:degree2-b} \\
    \forall x \exists y &: f_1(x,y), \label{fo:degree2-c} \\
    \forall x \exists y &: f_2(x,y), \label{fo:degree2-d} \\
    \forall x \forall y &: \xi(x,y) \Leftrightarrow (f_1(x,y) \vee f_2(x,y)), \label{fo:degree2-e} \\
    \forall x \forall y &: \neg f_1(x,y) \vee \neg f_2(x,y). \label{fo:degree2-f} 
\end{align}

We set the weights of all the predicates to $1$ except for $\xi$ for which we set $w(\xi) = w$ and $\overline{w} = 1$. Let $\Gamma$ be the conjunction of (\ref{fo:noloops}), (\ref{fo:symmetry}), (\ref{fo:degree2-a}), (\ref{fo:degree2-c}), (\ref{fo:degree2-d}), (\ref{fo:degree2-e}), (\ref{fo:degree2-f}). Since $\Gamma$ is in FO$^2$, we can use, e.g., the algorithm from \citep{beame2015symmetric} to compute WFOMC for any weight $w$ in time polynomial in the size of the domain and number of bits needed to encode $w$. Thus we can use the techniques described in Section \ref{sec:modelcountingfunction} to compute the MC-function $\operatorname{MC}_{\Psi,\Gamma,\Delta}$, where we set $\Psi = \{ \xi \}$. Finally, we still need to divide the MC-function by $2!^{|\Delta|} = 2^{|\Delta|}$ to account for the over-counting caused by $f_1$ and $f_2$. The number of 2-regular graphs is then equal to $\operatorname{MC}_{\Psi,\Gamma,\Delta}(\mathbf{n})/2^{|\Delta|}$ where $\mathbf{n} = 2|\Delta|$. For 3, 4, 5, 6, 7, 8, 9, 10 vertices this method yields the following numbers of 2-regular graphs:  1, 3, 12, 70, 465, 3507, 30016, 286884. One can check that these numbers are exactly the same as the numbers of 2-regular graphs listed in the {\em On-Line Encyclopedia of Integer Sequences} as sequence A001205.\footnote{http://oeis.org/A001205} We show one example of the MC-function divided by $2^{10}$ in Figure \ref{figure:mc-regular-graphs} for $|\Delta| = 10$. The $x$-coordinate of the red point shown there is $n = 2|\Delta| = 20$ and its $y$-coordinate $\operatorname{MC}_{\Psi,\Gamma,\Delta}/2^{10} = 286884$ is the number of 2-regular graphs on 10 vertices.

%

\begin{figure}
\centering
\includegraphics[width=0.75\linewidth]{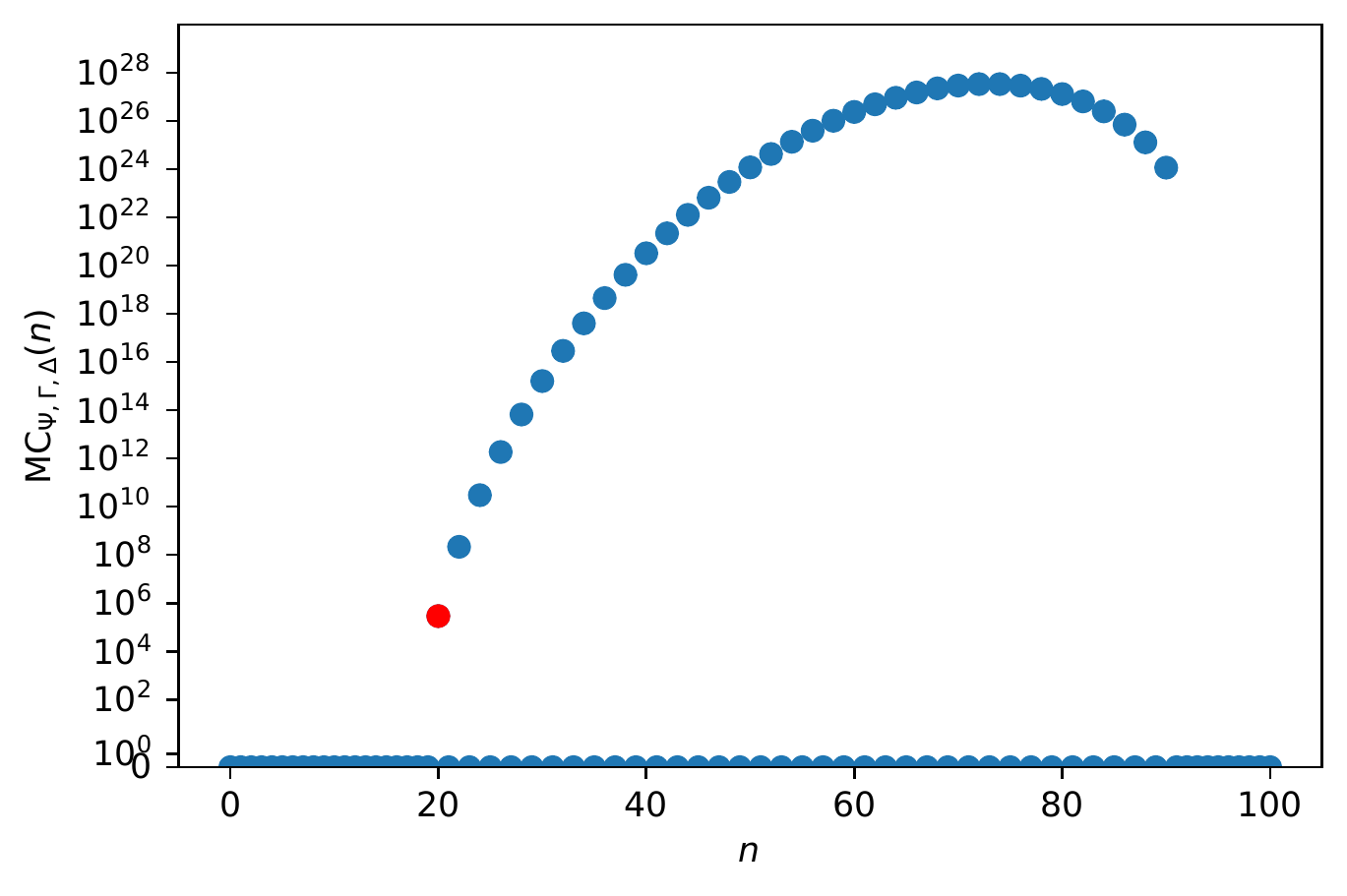}
\caption{The function $\operatorname{MC}_{\Psi,\Gamma,\Delta}(n)/2^{|\Delta|}$, where $\Psi$ and $\Gamma$ are described in the main text in Section \ref{sec:regular-graphs} and $|\Delta| = 10$.}\label{figure:mc-regular-graphs}
\end{figure}


One can easily adapt the above example for counting the number of $k$-regular graphs for a general $k$ (although the complexity of the encoding grows with $k$). We note that counting $k$-regular graphs is, in fact, an interesting, not completely trivial, problem; for $k = 2,3,4,5$ we refer to \citep{goulden1986labelled} for more details.

\subsection{The General Case}\label{sec:exists-k-general-case}

In this section we use the results from the previous sections to show that computing WFOMC of arbitrary two-variable first-order logic sentences with the quantifiers $\exists_{=k}$ is domain-liftable.

\begin{theorem}\label{thm:general-existsk}
Let $\Gamma$ be a sentence in the two-variable fragment of first-order logic, possibly containing a finite number of counting quantifiers $\exists_{=k_1}$, $\exists_{=k_2}$, $\dots$, $\exists_{=k_m}$. Then computing $\wfomc(\Gamma, w, \overline{w}, \Delta)$ is domain-liftable.
\end{theorem}

We already know from the previous sections how to do inference with special types of constraints, in particular with cardinality constraints and constraints of the form $\forall x \exists_{=k} y$ and $\exists_{=k} x \forall y$. The main difficulty in extending these results is to be able to encode sentences of the form $\forall x : A(x) \Leftrightarrow (\exists_{=k} R(x,y))$. We show how to do it in two steps. In Lemma \ref{lemma:lemma4}, we show how to encode sentences of the form $\forall x : A(x) \vee (\exists_{=k} R(x,y))$ using the form of constraints that we already know how to handle. We then use this intermediate result, together with Proposition \ref{prop:removing-negation} (Appendix A.2 in \citep{beame2015symmetric}), to encode the sentence  $\forall x : A(x) \Leftrightarrow (\exists_{=k} R(x,y))$ when we prove Theorem \ref{thm:general-existsk}.

\begin{lemma}\label{lemma:lemma4}
Let $R$ be a relation, $\Gamma$ be a first-order logic sentence and $\Upsilon$ be a conjunction of $\forall x \exists_{=k} y$ and $\exists_{=k} x \forall y$-constraints and cardinality constraints. Let us define $\Phi = \Phi_1 \wedge \Phi_2 \wedge \Phi_3 \wedge \Phi_4$ where: 
\begin{align*}
    \Phi_1 =& \forall x \exists_{=k} y : B^R(x,y),\\
    \Phi_2 =& (|U^R| = k), \\
    \Phi_3 =& \forall x \forall y : (A(x) \wedge B^R(x,y)) \Rightarrow U^R(y), \\
    \Phi_4 =& \forall x \forall y : \neg A(x) \Rightarrow (R(x,y) \Leftrightarrow B^R(x,y)).
\end{align*}
Then the following holds for WFOMC:
$$\wfomc(\Gamma \wedge \Upsilon \wedge (\forall x : A(x) \vee (\exists_{=k} y : R(x,y))), w,\overline{w}, \Delta) = \frac{1}{\binom{|\Delta|}{k}} \wfomc(\Gamma \wedge \Upsilon \wedge \Phi, w, \overline{w},\Omega_\textit{ext}),$$
where $\Omega$ is the set of all possible worlds on the domain $\Delta$ w.r.t.\ a given first-order logic language $\mathcal{L}$ and $\Omega_{\textit{ext}}$ is the set of all possible worlds on the domain $\Delta$ w.r.t.\ $\mathcal{L}$ extended by predicates $U^R$ and $B^R$, $\dots$, $f_k^R$ (in particular, we assume w.l.o.g.\ that $\mathcal{L}$ did not originally contain these predicates).
\end{lemma}
\begin{proof}
First, we show that for every possible world $\omega \in \Omega_\textit{ext}$ it holds: if $\omega \models \Phi$ then $\omega \models (\forall x : A(x) \vee (\exists_{=k} y : R(x,y)))$. For contradiction, let us assume that $\omega^* \in \Omega_\textit{ext}$ is a possible world such that $\omega^* \models \Phi$ and $\omega^* \not\models (\forall x : A(x) \vee (\exists_{=k} y : R(x,y)))$. Then there must be a $t \in \Delta$ such that $\omega^* \models \neg A(t) \wedge \neg (\exists_{=k} y : R(t, y))$. At the same time, it must be the case that $\omega^* \models \exists_{=k} y : B^R(t,y)$ (from $\Phi_1$) and $\omega^* \models \forall y : R(t,y) \Leftrightarrow B^R(t,y)$ (from $\Phi_4$). However, these cannot be all true at the same time, thus, we have arrived at a contradiction.

To finish the proof, let $[\omega]_\mathcal{L}$ denote the ``projection'' of $\omega$ on the language $\mathcal{L}$ which is the possible world obtained from $\omega$ by removing all atoms whose predicates are not contained in $\mathcal{L}$ (i.e.\ $U^R$ and $B^R$). 
We show that for every $\omega \in \Omega$ such that $\omega \models (\forall x : A(x) \vee (\exists_{=k} y : R(x,y)))$ there are exactly $\binom{|\Delta|}{k}$ possible worlds $\omega' \in \Omega_\textit{ext}$ such that $\omega' \models \Phi$ and $\omega = [\omega']_\mathcal{L}$. Let $\omega \in \Omega$ be any such possible world. Due to $\Phi_2$, to extend $\omega$, we have to select a set $\{ t_1, t_2, \dots, t_k \}$ of elements of the domain $\Delta$ and make $U(t_1)$, $U(t_2)$, $\dots$, $U(t_k)$ true (and no other). This can be done in $\binom{|\Delta|}{k}$ different ways. Once, we set the $U^R$ predicate in this way, there is only one way to extend the possible world $\omega$: For every $t \in \Delta$ such that $\omega' \models A(t)$, it must be true $\omega' \models B(t,t_1) \wedge \dots \wedge B^R(t,t_k)$. This is because, due to $\Phi_1$, for every $t \in \Delta$ there must be exactly $k$ domain elements $t_1'$, $\dots$, $t_k'$ such that $\omega' \models B^R(t,t_1') \wedge \dots \wedge B^R(t,t_k')$. Moreover, due to $\Phi_4$, if $\omega' \models A(t) \wedge B^R(t,t_i')$ are true then $\omega' \models U^R(t_i')$ must be true as well. However, there are only $k$ such domain elements $t_i'$ (due to $\Phi_2$). This means that $B^R(t,t_i')$ is uniquely determined. Moreover, for all the other $t \in \Delta$ such that $\omega' \models \neg A(t)$, $B^R$ must coincide with $R$ and hence is uniquely determined as well. This is what we needed to finish the proof.
\end{proof}

\noindent Now we are ready to prove Theorem \ref{thm:general-existsk}.

\begin{proof}[Proof of Theorem \ref{thm:general-existsk}]
We start by showing how to deal with sentences of the form  $\Gamma \wedge \Upsilon$ where 
$$\Upsilon = \forall x : A(x) \Leftrightarrow (\exists_{=k} y : R(x,y)).$$
First we rewrite the sentence 
$$\Upsilon_1 = \forall x : A(x) \Rightarrow (\exists_{=k} y : R(x,y))$$
into a form that we already know how to handle. 
For that we introduce a new unary predicate $B$ and define 
$$\Upsilon_1' = (\forall x : A(x) \Leftrightarrow \neg B(x)) \wedge (\forall x : B(x) \vee (\exists_{=k} y : R(x,y))).$$ 
Sentences in this form can already be handled by Lemma \ref{lemma:lemma4}. Let $\Omega_\textit{ext}'$ be the set of all possible worlds on domain $\Delta$ over the given language extended by the predicate $B$. Then for every possible world $\omega \in \Omega$ which is a model of $\Gamma \wedge \Upsilon_1$ there is exactly one possible world in $\Omega_\textit{ext}'$ which is a model of $\Gamma \wedge \Upsilon_1'$, and vice versa. 


Next we need to show how to handle the sentence $\Upsilon_2 = \forall x : A(x) \Leftarrow (\exists_{=k} y : R(x,y))$. At first this may seem difficult but Proposition \ref{prop:removing-negation} comes to the rescue here. Specifically, if we equivalently write  
$$\Upsilon_2 = \forall x : A(x) \vee \neg (\exists_{=k} y : R(x,y)),$$ 
we can get rid of the negation in front of $(\exists_{=k} y : R(x,y))$ using Proposition \ref{prop:removing-negation} as follows. We create two new unary predicates $C$ and $D$ and define  
$$\Upsilon_2' = \forall x : ((\exists_{=k} y : R(x,y)) \vee C(x)) \wedge (C(x) \vee D(x)) \wedge ((\exists_{=k} y : R(x,y)) \vee D(x))).$$ It follows from Proposition \ref{prop:removing-negation} that if we set $w(C) = \overline{w}(C) = w(D) = 1$ and $\overline{w}(D) = -1$ then for any sentence $\Theta$ it will hold
$$\wfomc(\Theta \wedge \Upsilon_2, w, \overline{w}, \Omega) = \wfomc(\Theta \wedge (\forall x : A(x) \vee C(x)) \wedge \Upsilon_2', w, \overline{w}, \Omega_\textit{ext}''),$$
where $\Omega$ is the set of all possible worlds on the domain $\Delta$ w.r.t.\ the given first-order logic language $\mathcal{L}$ and $\Omega_{\textit{ext}}''$ is the set of all possible worlds on the domain $\Delta$ w.r.t.\ $\mathcal{L}$ extended by the predicates $C$ and $D$.

Next we can apply consecutively the two steps described above which gives us 
$$\wfomc(\Gamma \wedge \Upsilon, w, \overline{w}, \Omega) = \wfomc(\Gamma \wedge (\forall x : A(x) \vee C(x)) \wedge \Upsilon_1' \wedge \Upsilon_2', w, \overline{w}, \Omega_\textit{ext}'''),$$
where $\Omega$ is the set of all possible worlds on the domain $\Delta$ w.r.t.\ a given first-order logic language $\mathcal{L}$ and $\Omega_{\textit{ext}}'''$ is the set of all possible worlds on the domain $\Delta$ w.r.t.\ $\mathcal{L}$ extended by the predicates $B$, $C$ and $D$. From Lemma \ref{lemma:lemma4}, we already know how to handle the sentences $\Upsilon_1'$ and $\Upsilon_2'$ when computing WFOMC.

Finally, we use the above to compute WFOMC of arbitrary FO$^2$ sentences which may contain quantifiers $\exists_{=k}$. This is relatively straightforward. Let $\Gamma$ be such a sentence. We just need to repeatedly replace every subformula of the form $\exists_{=k} y : \psi(x,y)$ by $A_\psi(x)$, where $A_\psi$ is a fresh unary predicate, and add a ``definition'' of this predicate in the form $(\forall x \forall y : B_\psi(x,y) \Leftrightarrow \psi(x,y)) \wedge (A_\psi(x) \Leftrightarrow (\exists_{=k} y : B_\psi(x,y)))$. It is not difficult to see that the resulting sentence (i) is in a form that can be handled by Theorem \ref{prop:existsk}, (ii) that it contains only two variables and (iii) that its size is independent of the size of the domain. Additionally, the only way in which the resulting sentence will depend on the domain is through the cardinality constraints of the form $|R| = k |\Delta|$ and, in any such a constraint, $k$ will grow only polynomially with the size of the domain (this is no problem for establishing domain liftability). The statement of the theorem then follows from the above and from Theorem \ref{prop:existsk}.
\end{proof}


\section{The Two-Variable Fragment with Counting}\label{sec:main-result}

In this section we show that the results from the previous sections imply domain liftability for the two-variable fragment of first-order logic with counting quantifiers.

\begin{theorem}\label{thm:main-theorem}
Weighted first-order model counting is domain-liftable for the two-variable fragment of first-order logic with counting quantifiers.
\end{theorem}
\begin{proof}
We can prove this theorem by reducing it to the case with just the $\exists_{=k}$-quantifier, whose domain-liftability is established by Theorem \ref{thm:general-existsk}. 

First, any sub-formula of the form $\exists_{\leq k} y : \psi(x,y)$ can be replaced by: 
$$(\forall y : \neg \psi(x,y)) \vee (\exists_{=1} y : \psi(x,y)) \vee (\exists_{=2} y : \psi(x,y)) \vee \dots \vee (\exists_{=k} y : \psi(x,y)).$$

\noindent Obviously, the size of the result of the above transformation is independent of the domain size and if the original formula was in FO$^2$, the new one will be in FO$^2$ as well.

Next we need to get rid of the sub-formulas of the form $\exists_{\geq k} y : \psi(x,y)$. Note that we cannot blindly apply the same method we used for the sub-formulas with the quantifier $\exists_{\leq k}$ because, in that case, the number of disjuncts in the resulting formula would grow with the size of the domain. Instead, we proceed as follows. We equivalently rewrite the sub-formula as:
$$\exists_{\geq k} y : \psi(x,y) = \neg \neg (\exists_{\geq k} y : \psi(x,y)) = \neg (\exists_{\leq k-1} y : \psi(x,y)).$$
We already know how to handle sub-formulas of this form. So we are done.
\end{proof}

\section{Related Work}\label{sec:related-work}

The work presented in this paper builds on a long stream of research in lifted inference \citep{poole2003first,braz2005lifted,broeck2011completeness,DBLP:conf/uai/GogateD11a,van2013lifted,van2014skolemization,beame2015symmetric,DBLP:conf/nips/KazemiKBP16,DBLP:conf/lics/KuusistoL18}. On the technical level, in the present paper, we directly exploit results from \citep{,van2014skolemization} that together established domain-liftability of the two-variable fragment of first-order logic, which we extended by allowing counting quantifiers. The result on domain-liftability of the two-variable fragment was relatively recently extended in two somewhat related directions that we describe below.

First, \citet{DBLP:conf/nips/KazemiKBP16} showed that so-called {\em domain-recursion rule}, which had been previously proposed by Van den Broeck, allows to enlarge the class of domain-liftable theories. In particular, they identified two new domain-liftable fragments of first-order logic which they call S$^2$FO$^2$ and S$^2$RU. These two classes contain among others certain theories with functionality axioms but, as also pointed out by \citet{DBLP:conf/lics/KuusistoL18}, not all FO$^2$ theories with functionality axioms are contained in them. The domain-liftable class identified in our work, i.e.\ FO$^2$ with counting quantifiers, and the classes studied by \citet{DBLP:conf/nips/KazemiKBP16} are incomparable. However, we believe that our techniques and theirs could be combined in future work.

More recently, \citet{DBLP:conf/lics/KuusistoL18} showed that WFOMC for FO$^2$ with at most one functionality constraint is domain-liftable using a rather complex argument. They also mentioned in the same work (although without giving any details) that if one could extend this result to multiple functionality constraints, domain-liftability of FO$^2$ with counting quantifiers would follow. Specifically, in the concluding section of their paper, they say: {\em It can be shown that WFOMC for formulae of two-variable logic with counting $\operatorname{C}^2$ can be reduced to WFOMC for $\operatorname{FO}^2$ with several functionality axioms.} Thus, in principle, our results from Section \ref{sec:exists-constraints}, where we established domain-liftability of FO$^2$ with an arbitrary number of functionality constraints, combined with their remark would also be sufficient to establish our main result. However, since there are no details and no proof, we had to provide these ourselves. In fact, we did not use reductions directly based on functionality constraints in our proofs (since that seemed to be rather wasteful), therefore we suspect that our reductions might also be more efficient. 

Finally, the methods that we used in the present paper resemble techniques used in enumerative combinatorics \citep{stanley1986enumerative}, in particular generating functions. We plan to investigate these connections more closely in future work.

\section{Conclusions}\label{sec:conclusions}

In this work we showed that the two-variable fragment of first-order logic with counting quantifiers is domain-liftable. This significantly broadens the class of weighted first-order model counting problems that can be solved in polynomial time. There is still a lot one can do from here, especially for improving the practical efficiency of lifted inference algorithms on problems that result from our reductions. 


\subsection*{Acknowledgements}

This work was supported by Czech Science Foundation project ``Generative Relational Models'' (20-19104Y), the OP VVV project {\it CZ.02.1.01/0.0/0.0/16\_019/0000765} ``Research Center for Informatics'' and a donation from X-Order Lab. The author is grateful to Guy Van den Broeck for pointing out the problem of inference in the two-variable fragment of FOL with counting quantifiers.

\appendix

\section{Computing Weighted Model-Counting Functions Using DFT}\label{sec:mcdft}

Here we describe an alternative approach to computing weighted model-counting functions based on discrete Fourier transform. This approach is a generalization of a method from our previous work \citep{kuzelka.complex} where it was used in the context of Markov logic networks.

\paragraph{Notation} We need a bit more notation here. We use $i$ to denote the imaginary unit $i^2 = -1$ and $\langle v, w \rangle$ to denote the inner product of the vectors $v$ and $w$ (when $v$ and $w$ are real vectors, inner product coincides with scalar product). 

\subsection{Discrete Fourier Transform}\label{sec:fourier}

Here we describe the basic properties of {\em multi-dimensional Fourier transform} (DFT) that will be needed in this paper. Let $d$ be a positive integer and let $\mathbf{N} = [N_1, \dots, N_d] \in (\mathbf{N} \setminus \{ 0 \})^d$ be a vector of positive integers. Let us define $\mathcal{J} = \{0,1,\dots, N_1 - 1 \} \times \{0,1,\dots, N_2 - 1 \} \times \dots \times \{0,1,\dots, N_d - 1 \}$. Let $f : \mathcal{J} \rightarrow \mathbb{C}$ be a function defined on $\mathcal{J}$. Then the DFT of $f$ is the function $g : \mathcal{J} \rightarrow \mathbb{C}$ defined as
\begin{equation}\label{eq:dft}
    g(\mathbf{k}) = \sum_{\mathbf{n} \in \mathcal{J}} f(\mathbf{n})  e^{ - i 2 \pi \langle \mathbf{k}, \mathbf{n} / \mathbf{N} \rangle }
\end{equation}
where $\mathbf{k}/\mathbf{N} \stackrel{def}{=} \left[ [\mathbf{k}]_1/N_1, [\mathbf{k}]_2/N_2, \dots, [\mathbf{k}]_d/N_d \right]$ (i.e.\ ``/'' denotes component-wise division). We use the notation $g = \mathcal{F}\left\{f\right\}$. The inverse transform is then given as
\begin{equation}\label{eq:idft}
    f(\mathbf{n}) = \frac{1}{\prod_{l=1}^d N_l} \sum_{\mathbf{k} \in \mathcal{J}} g(\mathbf{k})  e^{ i 2 \pi \langle \mathbf{n}, \mathbf{k} / \mathbf{N} \rangle}.
\end{equation}
It holds $f = \mathcal{F}^{-1} \left\{ \mathcal{F}\left\{ f \right\} \right\}$.

\subsection{Computing Weighted Model-Counting Functions Using DFT}

Let $\Omega$ be the set of all possible worlds on a given domain $\Delta$ and a given set of predicates $\mathcal{R}$. Here we show how to compute the WMC-function $\operatorname{WMC}_{\Psi,\Gamma,\Delta}(\mathbf{n})$ for given list of predicates $\Psi = (R_1,\dots,R_m)$ and a given sentence $\Gamma$ using DFT.

First, $\operatorname{WMC}_{\Psi,\Gamma,\Delta}(\mathbf{n})$ is a real-valued function of $m$-dimensional integer vectors. We can restrict the domain\footnote{Here, {\em domain} refers to the domain of a mathematical function, not to a {\em domain} as a set of domain elements $\Delta$.} of $\operatorname{WMC}_{\Psi,\Gamma,\Delta}(\mathbf{n})$ to the set $$\mathcal{D} = \{0,1,\dots, M_1 \} \times \{0,1,\dots, M_2\}\times \dots \times \{ 0, 1, \dots, M_m \}$$ 
where $$M_1 = |\Delta|^{\textit{arity}(R_1)},\; M_2 = |\Delta|^{\textit{arity}(R_2)},\; \dots, \; M_m = |\Delta|^{\textit{arity}(R_m)}.$$

Second, from the definition of DFT we then have for the Fourier transform $g(\mathbf{k})$:
\begin{equation}\label{dftq1}
    g(\mathbf{k}) = \sum_{\mathbf{n} \in \mathcal{D}} \operatorname{WMC}_{\Psi,\Gamma,\Delta}(\mathbf{n},w,\overline{w}) \cdot  e^{ - i 2 \pi \langle \mathbf{k}, \mathbf{n} / \mathbf{M} \rangle }
\end{equation}
where $\mathbf{k} = (k_1,\dots,k_m)$, $\mathbf{M} = (M_1+1,\dots,M_m+1)$ and the division in $\mathbf{n} / \mathbf{M}$ is again component-wise.

Third, let $\Omega^* = \{\omega \in \Omega | \omega \models \Gamma \}$. Let $\mathcal{R}$ be the set of all predicates $R$ that have non-neutral weights (i.e.\ $w(R) \neq 1$ or $\overline{w}(R) \neq 1$). 
Using the definition of $\operatorname{WMC}_{\Psi,\Gamma,\Delta}(\mathbf{n},w,\overline{w})$, we can write 
\begin{multline*}
    g(\mathbf{k}) = \sum_{\mathbf{n} \in \mathcal{D}} \left( \sum_{\omega \in \Omega^* : \mathbf{N}(\Psi,\omega) = \mathbf{n}}  \prod_{R \in \mathcal{R} } w(R)^{\mathbf{N}(R,\omega)} \cdot \overline{w}(R)^{|\Delta|^{\textit{arity}(R)}-\mathbf{N}(R,\omega)} \right) e^{ - i 2 \pi \langle \mathbf{k}, \mathbf{n} / \mathbf{M} \rangle } \\
    = \sum_{\mathbf{n} \in \mathcal{D}} \sum_{\omega \in \Omega^* : \mathbf{N}(\Psi,\omega) = \mathbf{n}} \left( \prod_{R \in \mathcal{R} } w(R)^{\mathbf{N}(R,\omega)} \cdot \overline{w}(R)^{|\Delta|^{\textit{arity}(R)}-\mathbf{N}(R,\omega)} \right)  e^{ - i 2 \pi \langle \mathbf{k} / \mathbf{M}, \mathbf{N}(\Psi,\omega)  \rangle } \\
    = \sum_{\omega \in \Omega^*}  \left( \prod_{R \in \mathcal{R}} w(R)^{\mathbf{N}(R,\omega)} \cdot \overline{w}(R)^{|\Delta|^{\textit{arity}(R)}-\mathbf{N}(R,\omega)} \right)  e^{ - i 2 \pi \langle \mathbf{k} / \mathbf{M}, \mathbf{N}(\Psi,\omega)  \rangle } \\
    = \wfomc(\Gamma, w', \overline{w}', \Delta),
\end{multline*}


\noindent where we set $w'(R_i) = w(R_i) \cdot e^{-2 i \pi k_i/M_i}$ for all $R_i \in \Psi$, $w'(R) = w(R)$ for all $R \in \mathcal{R} \setminus \Psi$ and $\overline{w}'(R) = 1$ for all $R \in \mathcal{R}$. Thus, we can compute the DFT of WMC-functions using a polynomial number (in $|\Delta|$) of queries to a WFOMC oracle. Once we have the DFT $g(\mathbf{k})$ of the MC-function, obtaining the MC-function from the DFT is trivial. We just compute the inverse DFT of $g(\mathbf{k})$. Importantly, to obtain this, we did not need to add explicit cardinality constraints, expressed as first-order logic sentences, to $\Gamma$ or modify the formulas in it or in the set $\Psi$ in any way. 



\paragraph{A note on representation of complex numbers} In \citep{kuzelka.complex} we discuss the issues of representing complex numbers in the computations such as DFT in detail. 

\section{Omitted Proofs}

In this section we give proofs that were omitted from the main text.

\subsection{Proof of Proposition \ref{prop:lagrange}}

First, we bound the coefficients of monomials of the polynomials $l_i(x)$. We write $l_i(x) = \sum_{j=0}^{d} \frac{e_{i,j}}{f_{i,j}} \cdot x^j$ where $e_{i,j}, f_{i,j} \in \mathbb{N}$. We have
\begin{equation}\label{eq:maxe}
    \max_{i} \log{|e_{i,j}|} \leq \max_i \log{\left( 2^d \prod_{\scriptsize{\begin{array}{c} 0 \leq j \leq d \\ i \neq j \end{array}}}^{d} x_j \right)} \leq \log{\left( 2^d \max_j |x_j|^d \right)} = d \log\left(2 \max_j |x_j|\right),
\end{equation}
\begin{multline}\label{eq:maxf}
    \max_i \log{|f_{i,j}|} \leq \max_i \log{\left( \prod_{\scriptsize{\begin{array}{c} 0 \leq j \leq d \\ i \neq j \end{array}}}^{d} |x_i - x_j| \right)} \leq \max_i \log{\left( \max_j |x_i - x_j|^d \right)} \\ 
    = d \log{\left( 2 \max_j |x_j| \right)}.
\end{multline}
    
\noindent Let $y_i = \frac{y_i'}{y_i''}$ where both $y_i'$ and $y_i''$ are integers. We then also have for the coefficient $a_j$ of the monomial $x^j$ in the interpolating polynomial: 
$$a_j = \sum_{i = 0}^d \frac{y_i'}{y_i''} \frac{e_{i,j}}{f_{i,j}} = \frac{\sum_{i = 0}^d y_i' e_{i,j} \prod_{0 \leq k \leq d, k \neq j} y_k'' f_{k,j}}{\prod_{i=0}^d y_i'' f_{i,j}}.$$
Both the numerator and denominator of the last expression are integers and it holds
\begin{multline*}
    \log \left( \left| \sum_{i = 0}^d y_i' e_{i,j} \cdot \prod_{0 \leq k \leq d, k \neq j} y_k'' f_{k,j} \right|\right) \leq \log \left( \left|d \max_{i} \left\{ y_i' e_{i,j} \right\} \cdot \prod_{0 \leq k \leq d, k \neq j} y_k'' f_{k,j} \right|\right) \\
    \leq \log d + \log \max_{i}|y_i'| + \log\max_{i} |e_{i,j}| + \sum_{0 \leq k \leq d, k \neq j} \left(\log |y_k''|  + \log |f_{k,j}| \right)\\
    \leq \log d + \log \max_{i}|y_i'| + d \log \max_i|y_i''| + \left(d+d^2\right) \log\left(2 \max_j |x_j|\right)
\end{multline*}
(here the last inequality follows from (\ref{eq:maxe}) and (\ref{eq:maxf})). It also holds 
\begin{equation*}
    \log \left( \left| \prod_{i=0}^d y_i'' f_{i,j} \right| \right) \leq d \log \max_i|y_i''| + d \max_{i} \log\left(|f_{i,j}| \right)
    \leq d \log \max_i|y_i''| + d^2 \log\left(2 \max_j |x_j|\right).
\end{equation*}
It follows that the number of bits needed to represent the coefficients of the interpolating polynomial grows only polynomially with the number of bits needed to encode the points $(x_i,y_i)$, which is what we needed to show.
\qed

\subsection{Proof of Proposition \ref{prop:wfomcbits}}

Let use define the following set of integer vectors 
$$\mathcal{D} = \{0,1,\dots, M_1 \} \times \{0,1,\dots, M_2\}\times \dots \times \{ 0, 1, \dots, M_m \}$$ 
where $M_1 = |\Delta|^{\textit{arity}(R_1)},\; M_2 = |\Delta|^{\textit{arity}(R_2)},\; \dots, \; M_m = |\Delta|^{\textit{arity}(R_m)}.$ 
It is obvious that the weight of any possible world $\omega \in \Omega$ can be only one of the form $\prod_{i = 1}^m w(R_i)^{n_i} \cdot \overline{w}(R_i)^{|\Delta|^{\textit{arity}(R_i)}-n_i}$ for some $(n_1,n_2,\dots,n_m) \in \mathcal{D}$. That means that there are only polynomially many, in $\Delta$, different weights of possible worlds and the WFOMC is their weighted sum. Specifically, the WFOMC can be written as

\begin{equation}\label{eq:wfomcassum}
\wfomc(\Gamma,w,\overline{w},\Delta) = \sum_{(n_1,\dots,n_m) \in \mathcal{D}} C_{(n_1,\dots,n_m)} \prod_{i = 1}^m \left( \frac{w'(R_i)}{w''(R_i)} \right)^{n_i} \cdot \left(\frac{\overline{w}'(R_i)}{\overline{w}''(R_i)} \right)^{|\Delta|^{\textit{arity}(R_i)}-n_i}     
\end{equation}

\noindent where $C_{(n_1,\dots,n_m)}\in \mathbb{N}$. It is easy to see that $C_{(n_1,\dots,n_m)} \leq 2^{m \cdot |\Delta|^A}$ and $n_i \leq |\Delta|^A$ where $A = \max_{R \in \mathcal{R}} \textit{arity}(R)$. Next we define 
\begin{align*}
    & D'_{(n_1,\dots,n_m)} = C_{(n_1,\dots,n_m)} \prod_{i = 1}^m  w'(R_i)^{n_i} \cdot \overline{w}'(R_i)^{|\Delta|^{\textit{arity}(R_i)}-n_i}\\
    & D''_{(n_1,\dots,n_m)} = \prod_{i = 1}^m w''(R_i)^{n_i} \cdot \overline{w}''(R_i)^{|\Delta|^{\textit{arity}(R_i)}-n_i}
\end{align*}
\noindent We have 
\begin{multline*}
    \log D'_{(n_1,\dots,n_m)} = \log C_{(n_1,\dots,n_m)} + \sum_{i = 1}^m n_i \log w'(R_i) + \sum_{i = 1}^m \left( |\Delta|^{\textit{arity}(R_i)}-n_i \right)  \log \overline{w}'(R_i) \\
    \leq m A \log 2 + 2 m |\Delta|^A \log M
\end{multline*}
and similarly also
\begin{equation*}
    \log D''_{(n_1,\dots,n_m)} = \sum_{i = 1}^m n_i \log w''(R_i) + \sum_{i = 1}^m \left( |\Delta|^{\textit{arity}(R_i)}-n_i \right)  \log \overline{w}''(R_i) 
    \leq 2 m |\Delta|^A \log M.
\end{equation*}
It follows that each of the summands in (\ref{eq:wfomcassum}) can be represented as a fraction where both the numerator and the denominator are represented using a polynomial number of bits in $|\Delta|$ and in $\log M$.

Finally, $\wfomc(\Gamma,w,\overline{w},\Delta)$ is a sum of $|\mathcal{D}|$ such fractions and $|\mathcal{D}|$ is also polynomial in $|\Delta|$. The statement of the proposition follows from this. 
\qed

\section{Additional Examples}\label{appendix:additional-examples}

\subsection{Counting Anti-Involutive Functions}\label{appendix:anti-involutive}

Here we provide another example in which we illustrate how the techniques presented in Section \ref{sec:exists-constraints}, where we introduced $\forall\exists_{=1}$-constraints, can be used to efficiently encode more complex examples without much additional complexity Although the techniques that we developed in the later sections (Section \ref{sec:exists-k-general-case} and Section \ref{sec:main-result}) are more general but we may often pay for this generality by computational speed. For that we may sometimes need to slightly just these techniques. We illustrate it on the case of anti-involutive functions.

We look at functions from $M = \{1,2,\dots,m\}$ to $N = \{1,2,\dots,n\}$. 
We say that a function $f : M \rightarrow N$ is {\em anti-involutive} if $f(f(x)) \neq x$ for all $x \in M$. We are interested in the problem of counting all anti-involutive functions from $M$ to $N$. For that we first define such functions in first-order logic with counting quantifiers and cardinality constraints (which could also be represented in this case using counting quantifiers):

\begin{align}
    & |M| = m, \label{ex:involutive:1} \\
    & \forall x : M(x) \Rightarrow (\exists_{=1} y : f(x,y)), \label{ex:involutive:2}\\
    & \forall x,y : \neg M(x) \Rightarrow \neg f(x,y), \label{ex:involutive:3} \\
    & \forall x,y : \neg f(x,y) \vee \neg f(y,x). \label{ex:involutive:4}
\end{align}

\noindent We also assume that $\Delta = N$. The models of this theory on the domain $\Delta$ must be anti-involutive functions from a set of size $m$ (rather than from the set $M$) to $\Delta$. This means that to obtain the number of anti-involutive functions from $M$ to $N$, we will need to divide the model count that we obtain by $\binom{n}{m}$.

First, we replace both (\ref{ex:involutive:2}) and (\ref{ex:involutive:3}) by:

\begin{align}
    & |f| = m,\label{ex:involutive:5}\\
    & \forall x \exists y : M(x) \Rightarrow f(x,y).\label{ex:involutive:6}
\end{align}

The correctness of this transformation follows from similar reasoning as we used in the proof of Lemma \ref{lemma:lemma1}. Now we have no more counting quantifiers, only first-order logic sentences and cardinality constraints. Hence, all we need is to compute the MC-function $\operatorname{MC}_{\Psi, \Gamma, \Delta}(\mathbf{n})$, where $\Psi = (f, M)$ and $\Gamma$ is a conjunction of (\ref{ex:involutive:4}) and (\ref{ex:involutive:6}), and then use it to count only over the possible worlds that satisfy the cardinality constraints $|f| = |M| = m$.

\begin{figure}[t]
\centering
\includegraphics[width=0.49\linewidth]{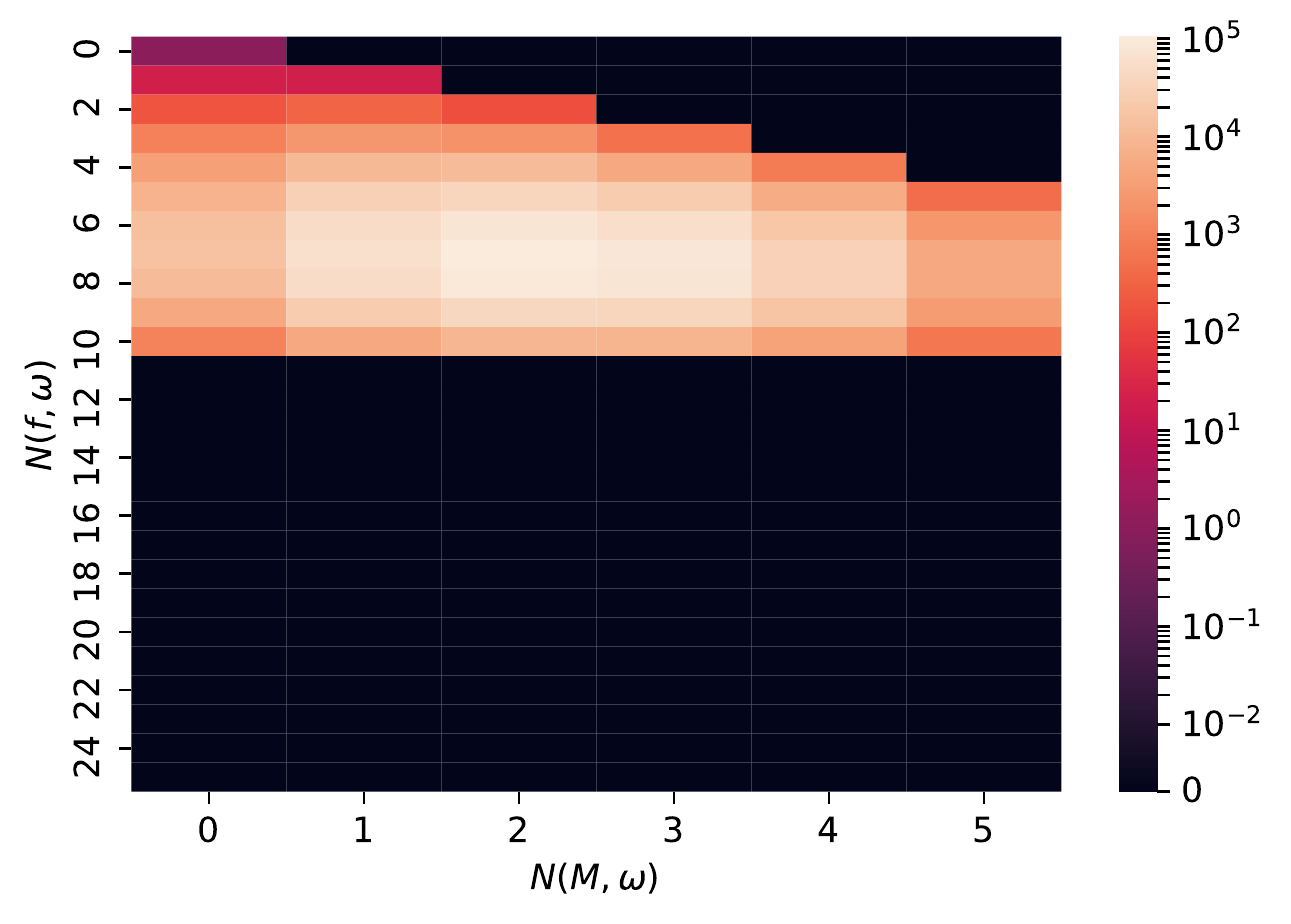}
\includegraphics[width=0.49\linewidth]{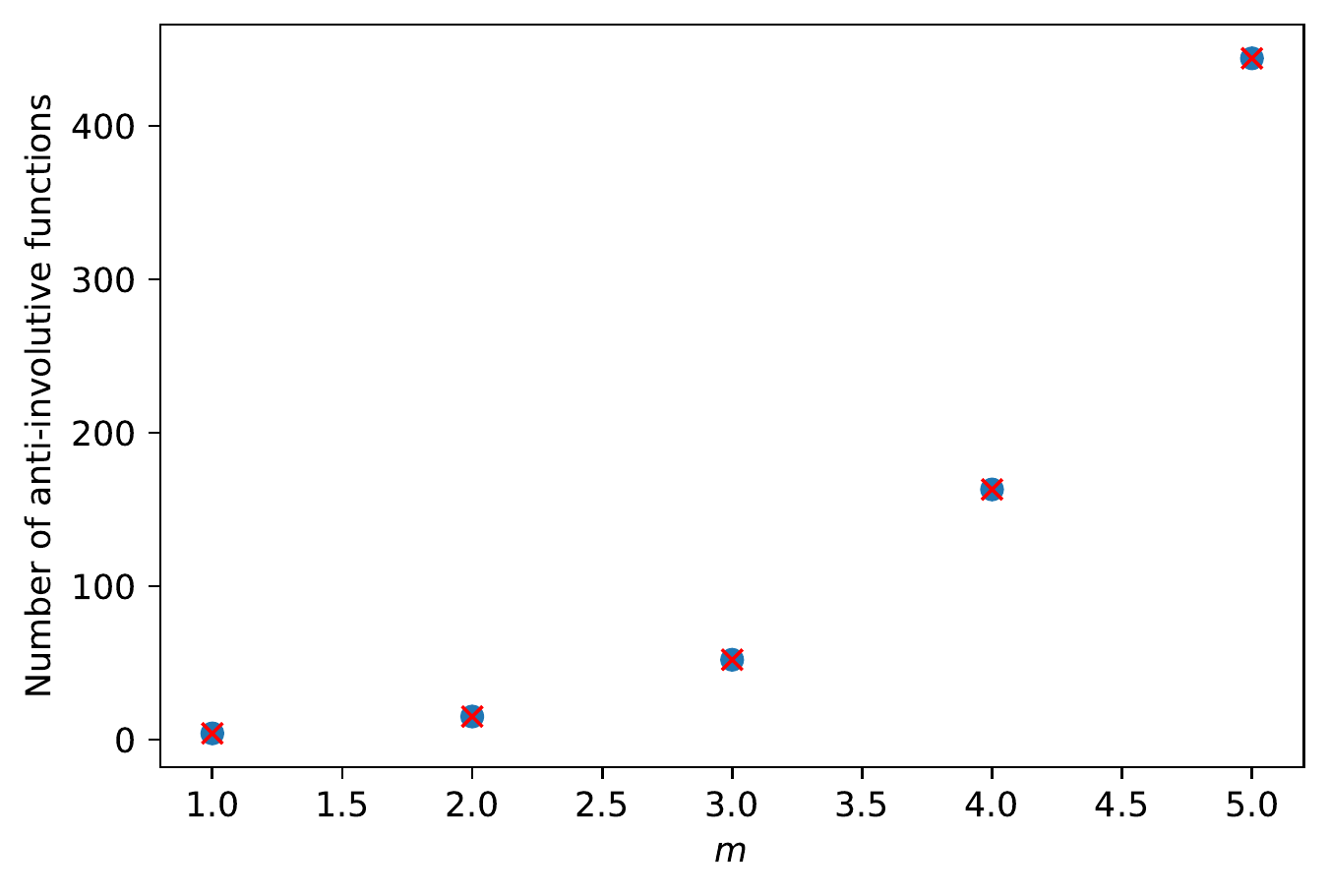}
\caption{{\bf Left:} The MC-function for computing the number of anti-involutive functions when $n = 5$. {\bf Right: } The number of anti-involutive functions from $\{1,2,\dots,m \}$ to $\{1,2,\dots,5 \}$.}\label{figure:anti-involutive-5}
\end{figure}

\begin{figure}[t]
\centering
\includegraphics[width=0.49\linewidth]{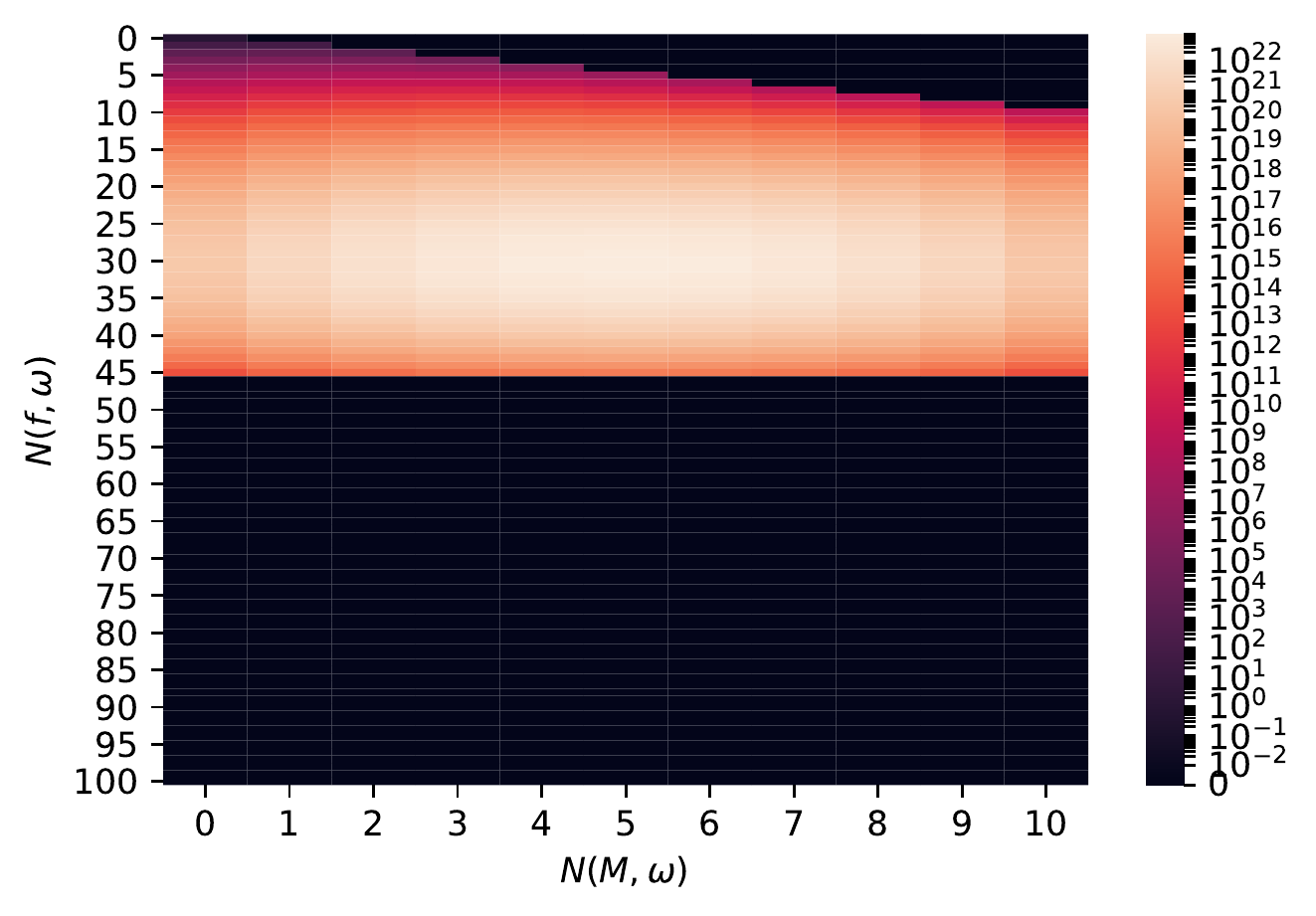}
\includegraphics[width=0.49\linewidth]{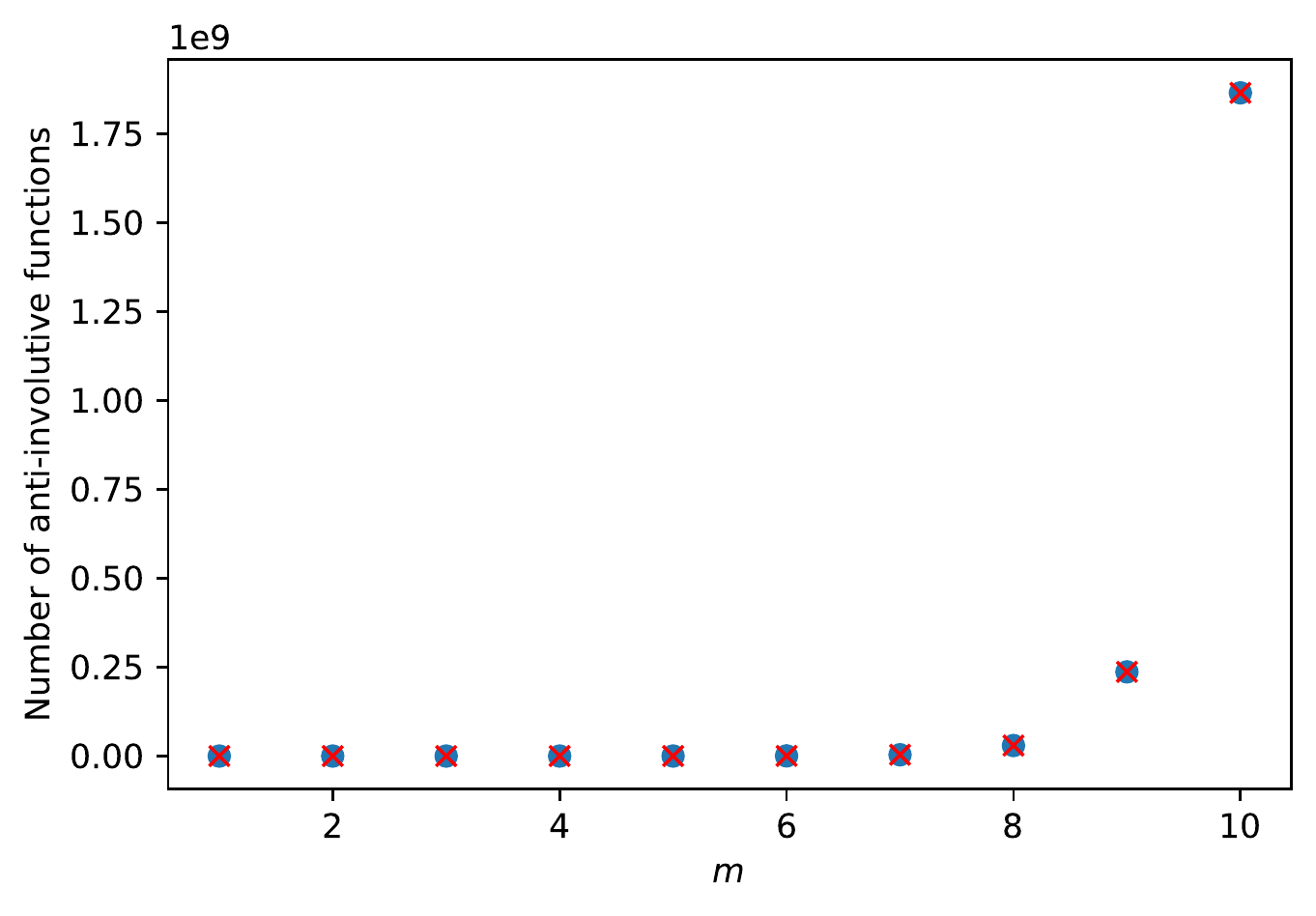}
\caption{{\bf Left:} The MC-function for computing the number of anti-involutive functions when $n = 10$. {\bf Right: } The number of anti-involutive functions from $\{1,2,\dots,m \}$ to $\{1,2,\dots,10 \}$.}\label{figure:anti-involutive-10}
\end{figure}

We plotted the resulting MC-functions in the left panels of Figure \ref{figure:anti-involutive-5} and Figure \ref{figure:anti-involutive-10}, for $n = 5$ and $n = 10$, respectively.\footnote{It is interesting to note that there is a good reason why the MC-function is zero in the roughly bottom half of the plots. Taking a closer look at the MC-function we can notice that it is zero for $|f| > \binom{n}{2}$. This is because of the constraint $\forall x \forall y : \neg f(x,y) \vee \neg f(y,x)$. Which implies that the cardinality of the relation $f$ cannot exceed the number of edges of the complete undirected graph on $n$ vertices.} In the right panels of these two figures, we plotted the numbers of anti-involutive functions computed by our approach (blue circles). Note that each of these plots corresponds to the diagonal of the respective MC-function (i.e.\ $\operatorname{MC}_{\Psi,\Gamma,\Delta}(m,m)$) divided by $\binom{n}{m}$.

As a sanity check we compared our results with the numbers given by the explicit formula
$$F(m,n) = \sum_{i=0}^{\lfloor m/2 \rfloor} (-1)^i (n-1)^{m-2i} \binom{m}{2i} \frac{(2i)!}{2^i (i!)},$$
derived by \citet{DBLP:conf/lics/KuusistoL18}. We plotted it as red crosses. As expected, both methods give the same results.

Alternatively, instead of replacing (\ref{ex:involutive:2}) and (\ref{ex:involutive:3}) by (\ref{ex:involutive:5}) and (\ref{ex:involutive:6}), we could have used the transformation from Lemma \ref{lemma:lemma4}. However, that would actually lead to a more complex encoding. So, even though, we would still be able to solve the counting problem in time polynomial in the size of the domain (i.e.\ in $n$), the exponent of the polynomial might be higher. This illustrates the fact that there may often be more efficient transformations than those we used in our proofs. Arguably, there seems to be quite some potential in investigating more efficient transformations for certain cases.



\bibliographystyle{named}
\bibliography{kr}

\end{document}